\documentclass[11pt]{article}
\pdfoutput=1

\usepackage[table,xcdraw,dvipsnames]{xcolor}  %
\usepackage{siunitx}
\usepackage{etoolbox}
\usepackage{color}

\usepackage[margin=2cm,includefoot,footskip=30pt]{geometry}
\usepackage[width=474.18663pt]{caption}

\usepackage[T1]{fontenc}
\usepackage[utf8]{inputenc}
\usepackage{lmodern}

\usepackage{graphicx}
\usepackage{subcaption}
\usepackage{tikz}
\usetikzlibrary{arrows, positioning, shapes.geometric}

\usepackage{csquotes}
\usepackage{lipsum}
\usepackage{framed}
\usepackage[most]{tcolorbox}       %
\tcbset{
  myquote/.style={
    colback=gray!10,      %
    colframe=gray!50,     %
    width=\linewidth,     %
    sharp corners,        %
    boxrule=0.5pt,        %
    left=2em,             %
    right=2em,            %
    top=1em,              %
    bottom=1em,           %
    enhanced jigsaw,      %
  }
}

\usepackage{url}
\usepackage{hyperref}
\usepackage{natbib}

\usepackage{booktabs}
\usepackage{multirow}
\usepackage{tabulary}
\usepackage{tabularx}

\usepackage{soul}
\usepackage{xspace}

\usepackage{algorithm}
\usepackage{algorithmicx}
\usepackage[noend]{algpseudocode}

\usepackage{amsmath,amssymb,amsfonts,amsthm}
\usepackage{thmtools,thm-restate}
\usepackage{mathtools}
\usepackage{nicefrac}
\usepackage{bm}
\usepackage{bbm}
\usepackage{dsfont}
\usepackage{upgreek}

\usepackage{enumitem}
\usepackage{array}

\usepackage{float}
\usepackage[bottom]{footmisc}
\usepackage{authblk}
\usepackage{microtype}
\usepackage{multicol}

\let\citet\cite

\definecolor{lightergray}{rgb}{0.9, 0.9, 0.9}
\definecolor{evenlightergray}{rgb}{0.95, 0.95, 0.95}
\definecolor{Darkblue}{rgb}{0,0,0.4}
\definecolor{Brown}{cmyk}{0,0.81,1.,0.60}
\definecolor{Purple}{cmyk}{0.45,0.86,0,0}
\hypersetup{colorlinks=true,
            citebordercolor={.6 .6 .6},
            linkbordercolor={.6 .6 .6},
            citecolor=Darkblue,
            urlcolor=blue,
            linkcolor=red}

\counterwithin{figure}{section}

\newtheorem{theorem}{Theorem}[section]
\newtheorem{example}[theorem]{Example}

\newtheorem{definition}[theorem]{Definition}

\newtheorem{fact}[theorem]{Fact}

\newtheorem{claim}[theorem]{Claim}

\renewcommand{\epsilon}{\varepsilon}

\renewcommand{\tilde}{\widetilde}

\DeclareMathOperator*{\Exp}{\mathbb{E}}

\newcommand{\EEs}[2]{\Exp_{#1}\left[#2\right]}

\newcommand{\reals}{\mathbb{R}}

\newcommand{\assignequals}{\coloneqq}

\newcommand{\1}{\mathds{1}}

\newcommand{\cA}{\mathcal{A}}

\newcommand{\cD}{\mathcal{D}}

\newcommand{\cQ}{\mathcal{Q}}
\newcommand{\cR}{\mathcal{R}}

\newcommand{\cS}{\mathcal{S}}

\newcommand{\cZ}{\mathcal{Z}}

\usepackage[capitalize,noabbrev]{cleveref}
\Crefname{fact}{Fact}{Facts}

\title{\textbf{The Limits of Preference Data for Post-Training}
}

\author[1,2]{Eric Zhao}
\author[2]{Jessica Dai}
\author[1]{Pranjal Awasthi}
\affil[1]{Google Research}
\affil[2]{University of California, Berkeley}

\date{}
\begin{document}

\allowdisplaybreaks
\maketitle

\begin{abstract}
Recent progress in strengthening the capabilities of large language models has stemmed from applying reinforcement learning to domains with automatically verifiable outcomes.
A key question is whether we can similarly use RL to optimize for outcomes in domains where evaluating outcomes inherently requires human feedback; for example, in tasks like deep research and trip planning, outcome evaluation is qualitative and there are many possible degrees of success.
One attractive and scalable modality for collecting human feedback is \textit{preference data}: ordinal rankings (pairwise or $k$-wise) that indicate, for $k$ given outcomes, which one is preferred.
In this work, we study a critical roadblock: preference data fundamentally and significantly limits outcome-based optimization.
Even with idealized preference data (infinite, noiseless, and online), the use of ordinal feedback can prevent obtaining even approximately optimal solutions.
We formalize this impossibility using voting theory, drawing an analogy between how a model chooses to answer a query with how voters choose a candidate to elect.
This indicates that grounded human scoring and algorithmic innovations are necessary for extending the success of RL post-training to domains demanding human feedback.
We also explore why these limitations have disproportionately impacted RLHF when it comes to eliciting reasoning behaviors (e.g., backtracking) versus situations where RLHF has been historically successful (e.g., instruction-tuning and safety training), finding that the limitations of preference data primarily suppress RLHF's ability to elicit robust strategies---a class that encompasses most reasoning behaviors.
\end{abstract}

\section{Introduction}
Preference data,  which takes the form of ordinal comparisons (e.g., answer $a_1$ is preferred to $a_2$ for query $q$), has been established as a highly effective and practical way of incorporating human feedback at scale \cite{DBLP:conf/nips/ChristianoLBMLA17}. Reinforcement Learning from Human Feedback (RLHF) \cite{DBLP:conf/nips/Ouyang0JAWMZASR22} and its variants \cite{rafailov2024directpreferenceoptimizationlanguage} have become the dominant paradigm for the instruction tuning, alignment, and safety training of large language models (LLMs) \cite{DBLP:conf/iclr/WeiBZGYLDDL22, DBLP:conf/nips/Ouyang0JAWMZASR22,bai2022constitutional}.
Preference data offers advantages in cost-effectiveness, scalability, and its ability to reward subtle qualitative improvements.

In contrast, the greatest advancements in reasoning, agentic behavior, and tool-use capabilities have been driven by reinforcement learning on verifiable rewards (``RLVR'') in domains where automated ground-truth verifiers are readily available \cite{deepseekai2025deepseekr1incentivizingreasoningcapability,o1-preview}, such as in coding, formal theorem proving, and close-ended exams \cite{DBLP:conf/nips/HendrycksBKABTS21,aime2024,white2024livebenchchallengingcontaminationfreellm,openai2025competitiveprogramminglargereasoning,ren2025deepseekproverv2advancingformalmathematical}.
However, many important applications, such as agentic research (e.g., Deep Research \cite{openai-deep-research-2025,google-gemini-deep-research-2024}) and creative writing \cite{guardian-chatgpt-creative-writing-2025, chakrabarty2025ai, ahuja2025finding}, lack binary or quantitative notions of success and require human feedback for evaluation.
A pressing challenge, therefore, is to determine how we can extend outcome-based optimization (e.g., via RL) to obtain capabilities in domains where human feedback is necessary.
Has RLHF fallen short of RLVR simply due to cost and the limitations of relying on preference data that is collected in an offline fashion?
Would sufficiently scaling up the collection of preference data allow for similar RL-driven breakthroughs?
\begin{figure}[htbp]
    \centering
    \begin{subfigure}[c]{0.49\linewidth}
        \centering
        \includegraphics[width=\linewidth]{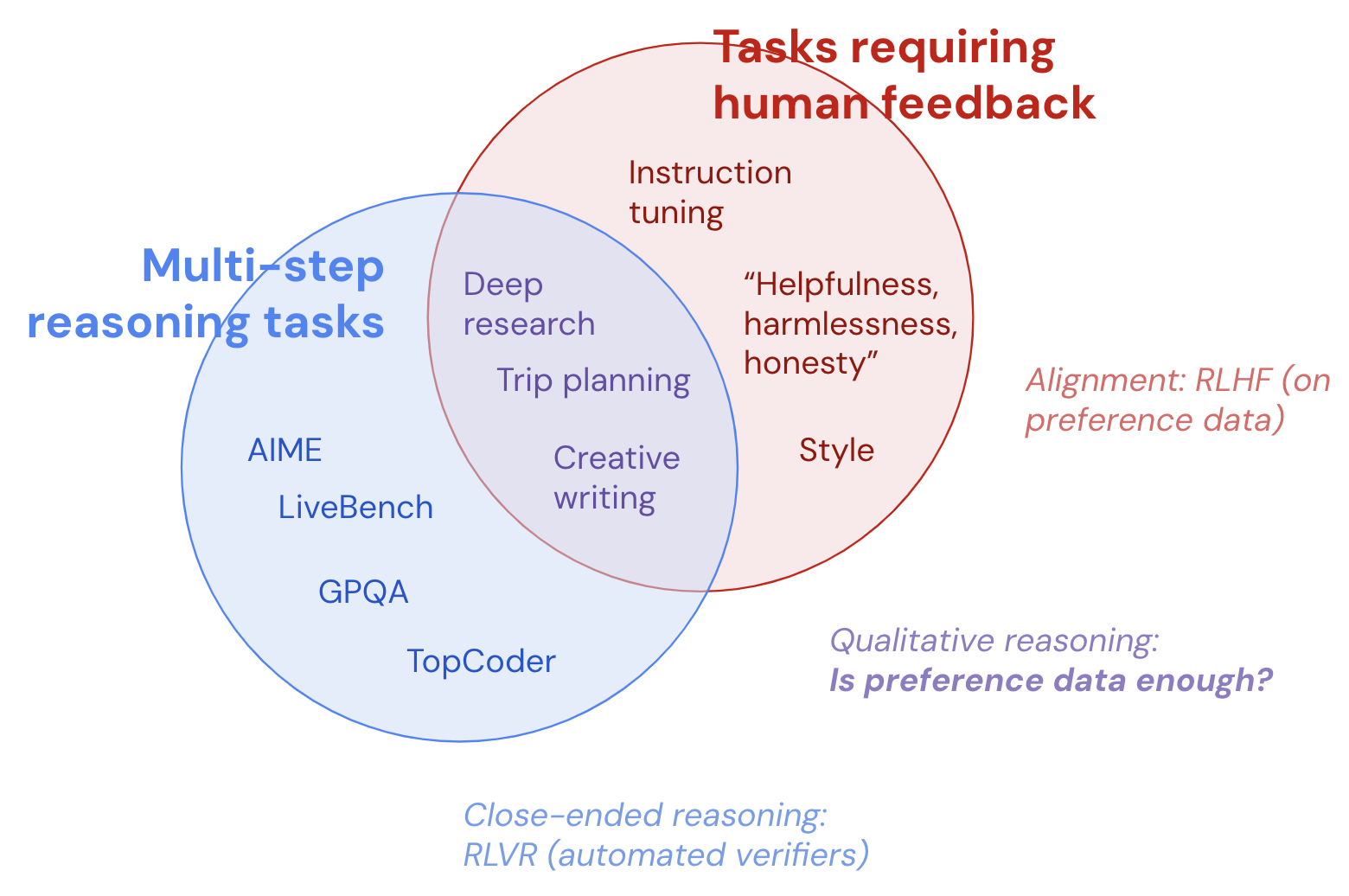}
        \label{fig:DiagramLeft}
    \end{subfigure}
    \hfill
    \begin{subfigure}[c]{0.5\linewidth}
        \centering
        \includegraphics[width=\linewidth]{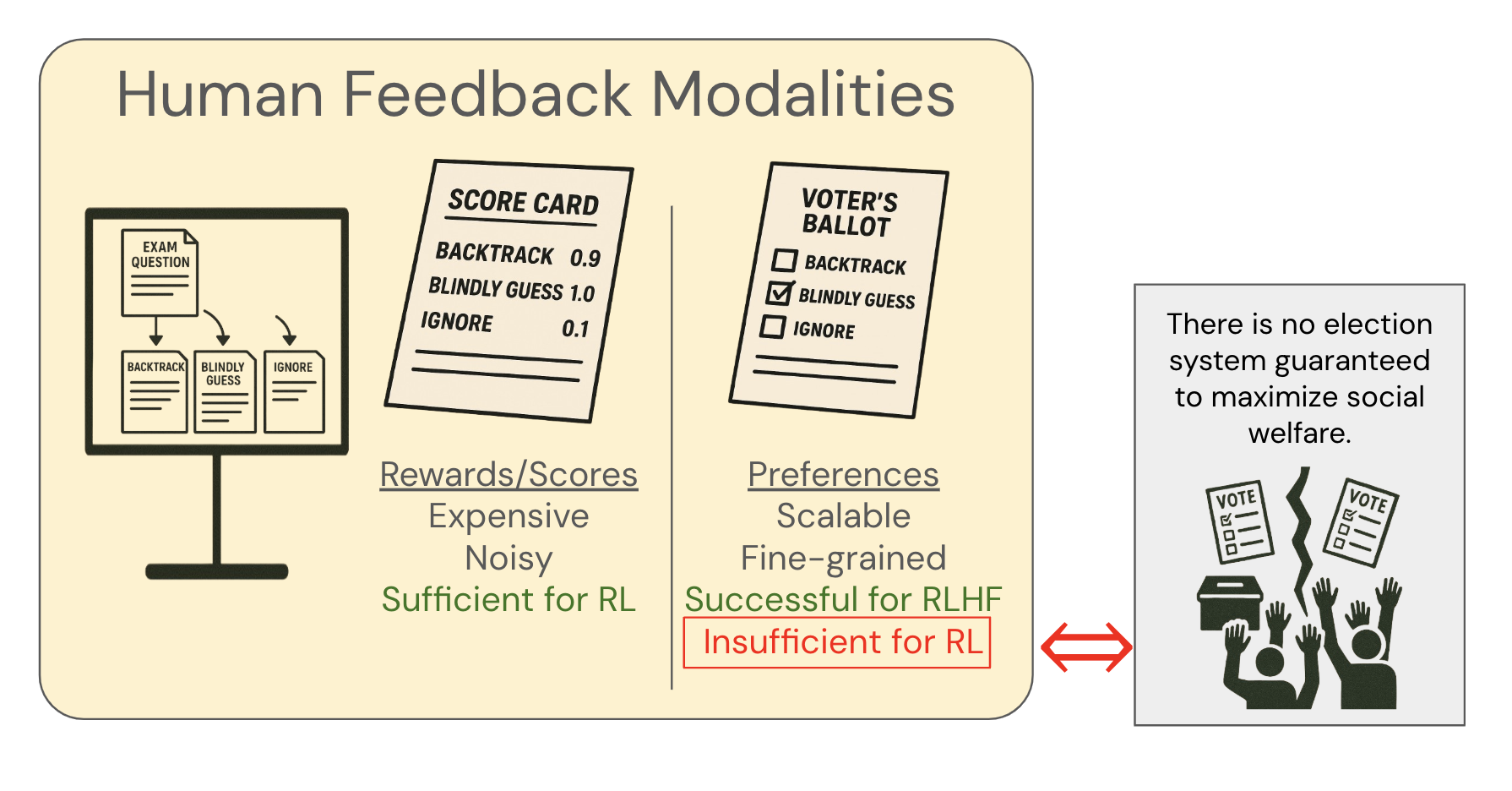}
        \label{fig:DiagramRight}
    \end{subfigure}
    \caption{\small Overview of our motivation and results. \textit{Left:} RLHF and RLVR have demonstrated empirical success on alignment and close-ended reasoning tasks, respectively. Our investigation of preference data is motivated by reasoning tasks that require human feedback. \textit{Right:} Comparison of scalar rewards vs. ordinal preferences as data modalities. Our impossibility result is due to a connection between the post-training of models using preference data and the analysis of electoral systems in social choice theory.}
    \label{fig:Diagram}
\end{figure}

We show that there is a fundamental challenge with relying on preference data for reinforcement learning, even when preference data is collected online and at scale.
We formalize this as an impossibility result (\Cref{theorem:bounded_computation_formal}): {given any pretrained model,} post-training purely on preference data can produce a significantly suboptimal post-trained model---even in idealized conditions where there is an endless volume of preference data available, the preference data is noiseless, unlimited computational power is allotted for optimization, and the preference data is collected in an online fashion.
We prove this result for a very general model of LLM post-training that formalizes post-training as learning how to route queries to a set of downstream circuits that were learned during pretraining; this model is informed by a growing body of empirical work \cite{wang2025reinforcementlearningreasoninglarge,yue2025doesreinforcementlearningreally} finding that post-training methods (whether RLHF, SFT, or RLVR) primarily reinforce existing circuits/capabilities rather than introducing wholly new ones.

This result also provides a complementary perspective on the limitations of RLHF, especially for reasoning. The poor performance of RLHF is typically attributed to the fact that preference data is collected offline, meaning that overfitting is a significant concern and it is necessary to regularize the post-trained model to not deviate significantly from its pre-trained base
\cite{ DBLP:conf/nips/Ouyang0JAWMZASR22,zhu2024ids}.
In contrast, our result demonstrates that RLHF has an additional, orthogonal, limitation: even if data \textit{collection} improves (e.g., online and at scale), the reliance on ordinal preference data as a \textit{format} means that RLHF can still fail to improve the model, and potentially even make the model significantly worse.
For reasoning in particular, as we show in Section \ref{sec:expts}, RLHF will tend to penalize and suppress noisy and robust chains-of-thought.
While this may not impact the effectiveness of RLHF for instruction-tuning and safety training, noisy and robust reasoning processes---such as the heavy use of backtracking---are key to improved reasoning capabilities and multi-step problem solving.
Our result underscores that there is no straightforward method to circumvent this issue unless one has precise prior knowledge of the specific desirable behaviors one hopes to elicit in the post-trained model and can tailor the feedback process accordingly---such as instructing labelers to explicitly avoid penalizing backtracking even when the backtracking is unnecessary.

From a technical perspective, our approach to proving Theorem \ref{theorem:bounded_computation_formal} is inspired by a novel analogy between the post-training of models and electoral processes. A rich body of social choice theory research has studied \textit{distortion}, which quantifies the extent to which electoral processes can be suboptimal with respect to social welfare \cite{procaccia2006distortion}.
Theorem \ref{theorem:bounded_computation_formal} can be seen as a novel distortion bound that applies to social choice problems that require a mechanism not only to select alternatives but also to determine how to effectively group voters into electoral districts. 

\subsection{Related works}
\label{subsec:related}
Our work seeks to bridge the empirical progress made in the LLM space and the theoretical insights about the limitations of ordinal data from social choice theory. We thus place our work in context with these communities as follows. 

\paragraph{Understanding RLHF.}
A wide range of recent work studies the extent to which RLHF and its variants can successfully learn from preference data. 
Commonly-observed limitations include failures to learn preference rankings (e.g., \cite{chen2024preference}), decreases in the log probability of preferred responses, (e.g., \cite{pal2024smaug, razin2024unintentional, zhu2024ids}), and general overoptimization (e.g., \cite{moskovitz2023confronting, liu2024provably}). 
These issues are often attributed to the distinction between on/off-policy algorithms or on/offline data collection.
In contrast, we study the preference data modality itself.

\citet{wang2023rlhf} was an early work that also studied preference data as a modality, demonstrating that there exists a pretrained model for which preference data is not sufficient to identify the optimal post-trained model.
Our results imply a strengthening of \citet{wang2023rlhf}'s non-identifiability result; in particular, we show that for \emph{any} pretrained model, preference data is not only insufficient for identifying the optimal post-trained model but cannot even identify a model whose utility is a constant fraction of being optimal.
\citet{wang2023rlhf} also suggests that preference data may be sufficient under specific preference noise assumptions, though we show that impossibility remains (\Cref{theorem:bounded_computation_formal_noise}) even in the case where preference noise follows a Bradley-Terry model with linear scores.

\paragraph{Reasoning capabilities.} Conventional wisdom is that methods for learning from preference data (e.g. DPO and RLHF) on outcomes alone typically fail for reasoning capabilities because such preference data over outcomes is too noisy---models could give mostly-correct reasoning traces but an incorrect answer, or arrive at a correct answer through faulty reasoning (e.g., as discussed in \cite{lai2024step}). Thus, much algorithmic work has focused on optimization over process rather than outcomes (e.g., \cite{lightman2023let, uesato_solving_2022, lai2024step}; see \citet{xu2025towards} for a survey). Notably, many of these methods still require access to a ground-truth verifier (e.g., \cite{xie2024monte, zhang2024chain, yuan2024advancing, wang2023making}); in fact, state-of-the-art reasoning models were developed in domains with such verifiers \cite{shao2024deepseekmath, deepseekai2025deepseekr1incentivizingreasoningcapability}.
This motivates our goal of understanding whether successful RL might be possible in ``reasoning'' domains that require human feedback, such as deep research, trip planning, or creative writing \cite{guardian-chatgpt-creative-writing-2025, chakrabarty2025ai, ahuja2025finding}. 

Finally, we note that connections between reasoning and robustness have been previously studied (e.g., \cite{yu2025benchmarking, zaremba2025trading}); our goal in this work is to use this connection to provide a more concrete answer to why vanilla RLHF fails for reasoning capabilities more than other applications for which RLHF is more successful, such as instruction tuning.

\paragraph{Social choice theory and distortion.}
Social choice has developed decades of work on \textit{distortion}, which measures the degree to which the outcome of an election can be suboptimal with respect to societal welfare \cite{procaccia2006distortion, anshelevich2021distortion}. The connection between social choice and learning from preference data for post-training has been widely discussed \cite{shirali2025directalignmentheterogeneouspreferences,dai2024mappingsocialchoicetheory, conitzer2024social, siththaranjan_distributional_2023}. 

Recently, \citet{yang-private-comm} showed that, when preference learning is performed in heterogeneous settings where labelers disagree, the need to reconcile disagreements between labelers introduces lower bounds on the distortion of any preference learning algorithm.
To the best of our knowledge, \citet{yang-private-comm} is the first to explicitly quantify the suboptimality of policies learned via RLHF in terms of distortion (in contrast to, e.g., the earlier non-identifiability results of \cite{wang2023rlhf}).
In contrast to \citet{yang-private-comm}, we are interested in showing unavoidable lower bounds for preference learning that hold generally, even in the absence of heterogeneity, and are thus a property of post-training rather than of resolving labeler disagreement.
To this end, we similarly build on the concept of distortion but apply it in a novel fashion, drawing an analogy between social choice and post-training. We discuss this more concretely at the end of \cref{sec:thy}.

\section{A formal perspective on post-training}
\label{sec:model}
Our starting point is to formalize a reasonable and general model of language model post-training.
The exact mechanisms behind the improvements seen from RLHF, RLVR and SFT is a topic of lively debate, but there is a growing consensus that these post-training methods reinforce existing capabilities attained in pretraining rather than produce new capabilities.
The current evidence includes that (1) distillation on reasoning models remains effective even when the example chains-of-thought are noisy and wrong \cite{muennighoff2025s1simpletesttimescaling,bespoke_stratos}; (2) reinforcement learning on a single sample works well \cite{wang2025reinforcementlearningreasoninglarge}; and (3) reasoning behaviors are already present in pretrained base models \cite{yue2025doesreinforcementlearningreally}.
Due to the diversity and scale of pretraining data and thus the capabilities represented in pretrained models, this is not an indictment of the power of post-training.
However, it does inform us that we can abstractly think of pretrained models as consisting of a set of capabilities, and post-training as updating how these capabilities are applied to queries.
In a more formal language, post-training can be stylized as learning a good way to route queries to appropriate downstream circuits (capabilities), where the circuits have already been learned in pretraining.

\paragraph{Formal model.}
We now turn to mathematically formulating our model of post-training, which we adopt to study the limitations of preference data.
Let us represent the universe of possible queries as $\cQ$ and the universe of possible responses as $\cR$.
We will use $\cD$ to denote a uniform distribution over queries $\cQ$.

We stylize a large language model as consisting of two parts.
The first part of the model is a set of circuits $\cS = \{s_1,\dots, s_m\}$ learned during pretraining, where each circuit $s_i : \cQ \to \cR$ helps compute potential responses to queries $q \in \cQ$.
The second part of the model is assigning which circuits should be used to answer which queries.
We formalize this as a mapping $\phi: \cQ \to \cZ$ of queries $q \in \cQ$ to internal representations $z \in \cZ$, and a mapping $g: \cZ \to \Delta_\cS$ of internal representations to circuits.
Here, $\Delta_\cS$ denotes the set of probability distributions over circuits $\cS$, and reflects the fact that we'll allow for the possibility that models randomly map queries to circuits.
We can therefore express a language model by the tuple $M = (\phi, g, \cS)$, and write its response to a query $q$ as $M(q) \coloneqq (g \circ \phi)(q)(q)$, where $(g \circ \phi)(q) \in \cS$ is the circuit we'll use to answer query $q$.

We stylize model post-training as improving how a pretrained language model $M_0 = (\phi_0, g_0, \cS_0)$ assigns queries to circuits, so as to maximize some utility function $u: \cQ \times \cR \to \reals$, where $u(q, r)$ denotes the utility\footnote{Utility can be thought of as a real-valued composite measure that quantifies, for example, correctness, readability, style, and other desiderata. This may not always be measurable from data, but assuming its existence is necessary to have any discussion about ``optimality.''} of responding to query $q$ with response $r$.
A post-training algorithm should return a language model $M = (\phi, g, \cS_0)$ with an improved assignment of queries to circuits, $g \circ \phi$, that results in a higher expected utility than its pretrained base $M_0$, i.e. $\EEs{q \sim \cD}{u(q, M(q))} 
 \geq \EEs{q \sim \cD}{u(q, M_0(q))}$.
We will use $\Phi \subset \cZ^\cQ$ to denote the set of possible mappings from queries to internal representations that our model class allows us to learn and that $\phi$ must thus be chosen from.

Finally, we note that even similar queries (e.g., variations of the same math problem) count as distinct elements of $\cQ$ in our model.
We are thus mainly interested in settings where the number of possible user queries $\cQ$ that we expect our language model to encounter is large.
This means that a language model should not always be able to identify the ideal circuit for processing a given query, i.e. $|\cQ| \gg |\cZ|, \log |\Phi|$.
For example, as we see later in \Cref{fig:deepseek_backtracks_pie}, language models do not always know immediately whether to answer a math problem using a circuit that is meticulous and aggressively backtracks or a circuit that tries to be clean and direct.
We also note that because language models do not have a dedicated circuit for answering each individual possible query, we work in the regime where $|\cQ| \gg |\cS|$.

\paragraph{Outcome-based optimization with preference data.}
With access to the utility function $u$ and pretrained model $M_0 = (\phi_0, g_0, \cS_0)$, one can always post-train a model $M = (\phi, g, \cS_0)$ to use the best possible strategy for assigning queries to pretrained circuits $\cS_0$, i.e.
\begin{equation}
\label{eq:optimal}
\EEs{q \sim \cD}{u(q, M(q))}
= \max_{M^* \in \{(\phi^*, g^*, \cS_0) \mid \phi^* \in \Phi,\; g^*: \cZ \to \Delta_\cS\}} \EEs{q \sim \cD}{u(q, M^*(q))}.
\end{equation}
Here, $\{(\phi^*, g^*, \cS_0) \mid \phi^* \in \Phi,\; g^*: \cZ \to \Delta_\cS\}$ is the set of all models that could be obtained from post-training $M_0$.
Standard reinforcement learning, as long as one has (even noisy) observations of the utilities $u$, will yield a post-trained model $M$ that satisfies \Cref{eq:optimal} in the limit of infinite time and samples.

We are interested in whether this still holds when, instead, only preference data is available---that is, labels are provided only in the sense that they compare  the relative utilities of distinct outcomes.
In particular, does there exist \emph{any} algorithm, be it RLHF or another RL-based approach, that is able to successfully post-train, but using ordinal preference data instead of cardinal reward feedback about utilities \(u\)?
If not, how close to an optimal post-trained model can one hope to get by learning from preference data?
Moreover, we want a guarantee that holds not just for specific pathological pretrained models, but for the post-training of \emph{any} pretrained model.
Note also that we will allow for the preference data to be queried online---this is to distinguish between the fundamental limitations of preference data and, for example, the overfitting caused by RLHF's usage of offline-collected preference data.

To formalize the preference learning setting, we consider access to a labeler that, given a query \(q \in \mathcal{Q}\) and examples of responses to $q$ generated by two downstream circuits \(s_i, s_j \in \mathcal{S}\), identifies which response was higher-utility. 
To isolate the impact of preference data as a modality, we allow for \emph{unlimited data}, and assume perfect (noiseless, unbiased) access to utility-based preference labels for every query $q \in \cQ$ and every pair of downstream circuits.
Concretely, we can write these preferences as orderings \(\succ_u \assignequals \{\succ_{q, u}\}_{q \in \cQ}\) where \(s_i \succ_{q, u} s_j\) if and only if \(u(q, s_i(q)) > u(q, s_j(q))\).
Note that we can interpret the preferences $\succ_u$ as either an oracle that returns preference comparisons on-demand or an infinitely large dataset containing all preference comparisons that one may hope to sample.

\section{Limits of preference data}
\label{sec:thy}
To show our main results, we consider an analogy between our model of post-training and the process of running a multi-district election.
In particular, we can analogize the set of circuits $\cS$ in our pretrained model as a set of candidates in an election (also referred to as ``alternatives''), and the set of possible user queries $\cQ$ as voters in the election.
Each query $q \in \cQ$ has a preference over the circuits $\cS$, determined by the utilities $\{u(q, s(q))\}_{s \in \cS}$ of applying each circuit to the query.

\subsection{Warm-up: A failure case for RLHF}
\label{subsec:warm-up}
Prior work studying the relationship between the RLHF algorithm and social choice has shown an equivalence between standard RLHF and a voting rule commonly known as \textit{Borda count} \cite{siththaranjan_distributional_2023}.

\begin{definition}[Borda count]
Let $\mathrm{rank}_{\succ_q}(s)$ denote the rank of $s$ in the total ordering $\succ_q$.
The Borda score of $s$ for query $q$ is given by $B_q(s) = m - \mathrm{rank}_{\succ_q}(s)$.
The Borda count voting rule selects $s_{\mathrm{Borda}} = \arg\max_{s \in \mathcal{S}} \sum_{q \in \mathcal{Q}} B_q(s)$, the alternative maximizing the sum of its Borda scores.
\end{definition}
Imagine we are given a silly pretrained model $\phi$ that is unable to distinguish between queries and must route them all to the same downstream circuit, such that the only flexibility in post-training is choosing this default circuit.
We can easily construct a scenario involving this silly pretrained model where Borda count---and equivalently, RLHF---is suboptimal.
Note that our silly choice of pretrained model reduces us to a  situation that resembles the heterogeneous/hidden-context settings studied in \cite{siththaranjan_distributional_2023,yang-private-comm}.
\begin{example}
\label{ex:borda}
Suppose $\cQ = \{q_1, q_2, q_3\}$, $\cS = \{s_A, s_B, s_C\}$, and $\cZ = \{z\}$. In the language of our model, this corresponds to a world where there are three possible queries, three possible downstream circuits, and a limited representational capacity where the model is unable to distinguish between queries and thus maps them to the same representation $z$. 
Let $\alpha \geq 0$ and $\beta \in \reals$ such that $2\alpha < \beta$.
Assign the following utilities: 
\[
\begin{array}{c|ccc}
& s_A & s_B & s_C \\
\hline
q_1 & 1 & 0 & 1 - \alpha \\
q_2 & 1 & 0 & 1 - \alpha \\
q_3 & 0 & 1 & \beta
\end{array}
\]
The Borda count winner is $s_A$, with $\sum_q B_q(s_A) = 4$, $\sum_q B_q(s_B) = 2$, and $\sum_q B_q(s_C) = 3$.
However, the total utility provided by each strategy is $\sum_q u(q, s_A(q)) = 2$, $\sum_q u(q, s_B(q)) = 1$, and $\sum_q u(q, s_C(q)) = 2-2\alpha + \beta$. Thus, while $s_C$ is the highest-utility strategy, Borda count selects $s_A$.
\end{example}
In Example \ref{ex:borda}, the circuit $s_C$ is never the winner when considering each query independently; thus, Borda count's ranking-based scores result in $s_C$ being dispreferred to $s_A$, even though it achieves the highest utility overall. In the language of social choice, $s_C$ is a ``compromise candidate.''
If it was known ahead of time that there were particular features associated with high-utility outcomes (i.e., those produced by $s_C$), one ad-hoc solution to fix this problem might be to instruct labelers to explicitly reward those features. However, post-training should ideally be able to learn successful strategies even when such strategies are unknown or unidentifiable \textit{a priori}, i.e. optimize for outcomes. 

\subsection{A more general lower bound}
We'll now show that the limitations of learning from preferences is fundamental, rather than the eccentricity of a particular choice of pretrained model or preference learning algorithm (e.g., RLHF as in Example \ref{ex:borda}).
In particular, we will show that for \emph{any} preference learning algorithm and any pretrained model, there are situations in which the algorithm must produce a severely suboptimal post-trained model.
Moreover, we will quantify just how suboptimal the post-trained model can be.
As we formalize in the following theorem, this suboptimality actually grows with the complexity of one's pretrained model.

\begin{restatable}{theorem}{origdistortion}
\label{theorem:bounded_computation_formal}
Consider any pretrained model \(M_0 = (\phi_0, g_0, \cS_0)\) and post-training algorithm $\cA$.
There always exists a post-training objective, i.e. a utility $u: \cQ \times \cR \to \reals$ we wish to maximize, such that: if we post-train $M_0$ on noiseless preference data $\succ_u$, the resulting model $M = \cA(M_0, \succ_u)$ is suboptimal by at least a multiplicative factor compared to the best model $M^*$ that we could have post-trained from $M_0$:
\begin{align}
\label{eq:distortion}
\underbrace{\max_{M^* \in \{(\phi^*, g^*, \cS_0) \mid \phi^* \in \Phi, g^*: \cZ \to \cS_0\}} \frac{\EEs{q \sim \cD}{u(q, M^*(q))}}{\EEs{q \sim \cD}{u(q, {M}(q))}}}_{(\mathrm{Distortion})}
\geq 
\Omega\Big(\sqrt{|\cS_0|}\Big)
\end{align}
when $|\cQ| \gg |\cS|, |\cZ|$.
Moreover, this lower bound holds even if we limit ourselves to situations where utilities are bounded in $[0, 1]$ or, even stronger, where $\sum_{s \in \cS_0} u(q, s(q)) = 1$ for all $q \in \cQ$.
\end{restatable}

The left-hand side of \Cref{eq:distortion} is known as distortion \cite{procaccia2006distortion} and, in our setting, measures the suboptimality of the preference learning algorithm $\cA$.
Our distortion lower bound is increasing in $|\cS_0|$, which is the number of circuits in one's pretrained model and thus a measure of its complexity.
The condition $|\cQ| \gg |\cS|, |\cZ|$ simply rules out unrealistic regimes where models have a dedicated circuit for processing every possible user query or models are able to exactly route each possible user query to the circuit that is individually ideal for them.
The full statement of \Cref{theorem:bounded_computation_formal} in \Cref{appendix:proofs} (\Cref{theorem:bounded_computation_formal_full}) provides a more general distortion bound independent of this condition.
We defer the proof of Theorem \ref{theorem:bounded_computation_formal} to Appendix \ref{appendix:proofs} as well.

At a high-level, we show that a difficult post-training objective must always exist using discrepancy upper bounds.
Given any pretrained model, we can carefully group possible user queries $\cQ$ into $k$ groups, $\cQ_1, \dots, \cQ_k$, such that---from the perspective of any post-trained model---the query groups should ``look'' rather similar.
That is, consider the distribution over pre-trained circuits induced by sampling queries from a group $\cQ_i$ and asking the post-trained model to assign the queries to circuits---this distribution over circuits must be similar in distance to the analogous distribution for any other query group $\cQ_j$.
This step of the argument follows by an application of the probabilistic method to an appropriate balls-and-bins analogy of the post-training problem.
To conclude our proof, we simply need a utility function to demand that queries from different groups $\cQ_i$ be assigned to different circuits; we appeal to existing constructions in social choice theory for this, in particular, a construction of \cite{boutilier2012optimal}.

\paragraph{Inferring cardinal data from noise.}
At first blush, Theorem \ref{theorem:bounded_computation_formal} may seem to contradict intuition from prior works about the optimality of RLHF (e.g., as suggested by the positive results in \cite{wang2023rlhf} and \cite{zhu2023principled}). 
These works study a setting where 
it is possible to 
extrapolate the degree of (average/overall) labeler preference by the noisiness of their labels. 
If labelers consistently indicate that they prefer a response $a_1$ over $a_2$ with a probability that is exponential in the difference between the utilities of $a_2$ and $a_1$, the \emph{variance} of the preferences provides enough information to learn an optimal policy.

However, there are several considerations that limit the degree to which these assumptions can be relied on in practice---especially as, for the most part, noise models from Bradley-Terry are unverifiable. 
First, it is unclear when preference noise can be assumed to provide a consistent indicator of preference strength. For example, as we explore in Section \ref{sec:expts}, distortion disproportionately affects the ability of RLHF to reinforce reasoning capabilities, but it seems unlikely in practice for noise to be informative in this setting.
Second, as noted by \cite{wang2023rlhf}, even when preference noise and preference strength follow some consistent relationship, one would need to know \textit{a priori} what that relationship is. 
Third, even if that relationship \textit{was} known a priori, there are natural noise models---even those satisfying the Bradley-Terry model---where preference data remains insufficient. 
For example, positive results from prior works assumed Bradley-Terry with exponential score functions---that is, $s_i \succ_{q, u} s_j$ with probability proportional to 
\(
\tfrac{\exp(u(q, s_i(q)))}{\exp(u(q, s_j(q)))+\exp(u(q, s_i(q)))}\).
As we show below in Theorem \ref{theorem:bounded_computation_formal_noise}, if we instead had a Bradley-Terry model with linear score functions, i.e. $s_i \succ_{q, u} s_j$ with probability proportional to \(
\tfrac{u(q, s_i(q))}{u(q, s_j(q))+u(q, s_i(q))}\), preference learning can still result in poor post-training.

\begin{restatable}{theorem}{distortion}
\label{theorem:bounded_computation_formal_noise}
Consider any pretrained model \(M_0 = (\phi_0, g_0, \cS_0)\) and post-training algorithm $\cA$.
There always exists a post-training objective, i.e. a utility $u: \cQ \times \cR \to \reals$ we wish to maximize, such that the following holds.
Even if we post-train $M_0$ on  preference data $\succ_u$ with noise consistent with the Bradley-Terry model with linear scores, the post-trained model $M = \cA(M_0, \succ_u)$ we obtain is suboptimal by at least a multiplicative factor compared to the best model $M^*$ that we could have post-trained from $M_0$:
\[
\max_{M^* \in \{(\phi^*, g^*, \cS_0) \mid \phi^* \in \Phi, g^*: \cZ \to \cS_0\}} \frac{\EEs{q \sim \cD}{u(q, M^*(q))}}{\EEs{q \sim \cD}{u(q, {M}(q))}} \geq  \Omega\Big({|\cS_0|}\Big)
\]
when $|\cQ| \gg |\cS|, |\cZ|$.
This lower bound holds even if we limit ourselves to situations where utilities are bounded in $[0, 1]$.
\end{restatable}
Note that \Cref{theorem:bounded_computation_formal_noise}'s lower bound is a factor of $\sqrt{|\cS_0|}$ larger (i.e. tighter) than that of \Cref{theorem:bounded_computation_formal}, even though the learning algorithm is afforded more information.
This is because the distortion bound in \Cref{theorem:bounded_computation_formal} holds even when utilities are restricted to sum to 1 for each query, i.e. $\sum_{s \in \cS_0} u(q, s(q)) = 1$ for all $q \in \cQ$.
In contrast, in order to handle the fact that we allow our preference learning algorithms to enjoy Bradley-Terry noise, \Cref{theorem:bounded_computation_formal_noise} instead constructs situations where some queries are inherently harder than others, meaning that utilities do not necessarily sum to 1; this ends up to also allow us to obtain a lower bound on distortion that grows faster with $|\cS_0|$.
The full statement of \Cref{theorem:bounded_computation_formal_noise} in \Cref{appendix:proofs} (\Cref{theorem:bounded_computation_formal_noise_full}) provides a more general distortion bound independent of the condition that $|\cQ| \gg |\cS|, |\cZ|$.

\paragraph{Beyond distortion.}
We have shown that true outcome-based optimization cannot be guaranteed with preference data.
While distortion theory is helpful for proving this limitation, the restrictions of the theory also provides prescriptive value in suggesting some practical workarounds.

One restriction of the theory is that it assumes that utilities are fine-grained.
One workaround is to explicitly instruct human labelers to coarsen their feedback, and ``tie'' responses by default unless one is significantly better than another.
This is not a general solution, however, and would only work for a narrow set of problems where, e.g. binary, notions of success suffice.

Another restriction of the theory is that it assumes only preference data is available.
This leads us to perhaps the most promising workaround: even a small amount of cardinal data (real-valued reward feedback) is sufficient to fill in the gaps left by the use of preference data.
For example, distortion can be effectively avoided with a minimal amount of cardinal data---specifically, logarithmically as much as the amount of preference data.
This cardinal data may take the form of flagging especially poor responses, or specifying the magnitude of a particular preference.
\citet{amanatidis2021peeking} shows that with only $\propto \log^2 |\cS|$ cardinal queries, it is possible to design a voting rule with at most constant $O(1)$ distortion---whereas there is a distortion lower bound of $\Omega(\sqrt{|\cS|})$ without cardinal data.
\begin{theorem}[Corollary 2 of \cite{amanatidis2021peeking}]
    There is a preference learning algorithm, Acceptable Range Voting, that guarantees a constant upper bound on distortion using only $\propto \log^2 |\cS|$ cardinal queries.
\end{theorem}
\paragraph{Connection to existing distortion results.}
We conclude this section by highlighting two key distinctions between prior distortion work and our model. 
One subtlety is that, under our model, it is possible to route different queries to different downstream circuits---that is, $(g \circ \phi)(q)$ can map to different distributions over $\cS$ for different queries---and thus our model of post-training is not directly analogous to an ordinary election. 
Rather, our results are for the scenario where there is both a fixed set of alternatives and a fixed set of ways to draw electoral districts, and the goal is to pick a social welfare-maximizing way of dividing voters into electoral districts and choosing an alternative for each district. To the best of our knowledge, this setting has not been explicitly studied in prior distortion work. 
Second, in prior works, the analogy between RLHF and elections lies in modeling individual annotators as voters, and distortion arises due to heterogeneity or disagreement across annotators (as in \cite{yang-private-comm}). In contrast, under our model, which instead analogizes voters to queries, suboptimality is unavoidable when training any model of finite computational complexity.
Our findings suggest that the limitations of ordinal data extend beyond reconciling diverse beliefs---situations that mirror elections literally---to any scenario constrained by finite inference-time compute, even in the absence of subjective disagreement across human annotators.

\section{Reasoning as a case study}
\label{sec:expts}
In this section, we explore why the limitations of preference data seem to disproportionately impact the effectiveness of RLHF in eliciting reasoning capabilities, as opposed to the successful applications of RLHF like safety training and instruction tuning. 
The problems we have access to ground-truth verifiers for are typically close-ended reasoning tasks with binary notions of success---for instance, AIME problems or coding exercises like Leetcode.
The current generation of reasoning models, which are post-trained for these tasks, are well-known to exhibit a specific set of ``cognitive'' behaviors: backtracking, error correction, verification, intentional explicitness, rigor, and planning \cite{gandhi2025cognitive}. 

We argue that these behaviors can be understood as \textit{robustness}, and, {conversely}, that robustness should be understood as a learned strategy for handling reasoning problems \cite{yu2025benchmarking, zaremba2025trading}.
In the language of our model from Section \ref{sec:model}, these reasoning strategies can be thought of corresponding to particular circuits in $\cS$ that correspond to ``robust'' behaviors, and which we want to reinforce for improved reasoning. 
As we showed in Example \ref{ex:borda}, the primary failure mode of Borda count (and thus RLHF) is the existence of compromise candidates; for reasoning, ``robust'' strategies can be considered to be compromise candidates that do well across the universe of queries, but often lose when examining a single query at a time. 
Our results therefore provide an additional explanation as to why RLHF appears to be especially limited when it comes to eliciting improved reasoning.
We defer interested readers to Appendix~\ref{app:expdetails} for experiment details and Appendix~\ref{app:addfigs} for supplementary figures.

\subsection{Robustness as reasoning strategy}
\label{sec:reasoning_as_robustness}

Existing works have gathered strong evidence that reasoning models are less susceptible to malicious user queries, especially with heavier use of inference-time compute \citep{zaremba2025trading}.
We show that this is part of a broader phenomenon where reasoning models are generally more robust---including to corruptions and perturbations to their \emph{chains-of-thought}.

In Figure \ref{fig:model_comparison}, we show that finetuning a non-reasoning model on reasoning traces improves robustness to three types of chains-of-thought perturbations.
The non-reasoning model (Qwen2.5-7B-Instruct) is significantly more sensitive to perturbations, with all three perturbation types resulting in substantial accuracy drops.  In contrast, a reasoning model (DeepSeek-R1-Distill-Qwen-7B) distilled onto the same non-reasoning model (Qwen2.5-7B-Instruct) exhibited significantly higher resilience---including nearly no degradation from the insertion of random \texttt{not} tokens.
This experiment was performed on problems from LiveBench Reasoning, AIME, and MATH \cite{white2024livebenchchallengingcontaminationfreellm, aime2024,DBLP:conf/nips/HendrycksBKABTS21}.

\begin{figure}[htbp]
  \centering
    \begin{subcaptionbox}{\small Comparison of accuracy between Qwen2.5-7B-Instruct (non-reasoning) and DeepSeek-R1-Distill-Qwen-7B (reasoning) under different output perturbations.
  \label{fig:model_comparison}}[0.48\textwidth]
    {\includegraphics[width=\linewidth]{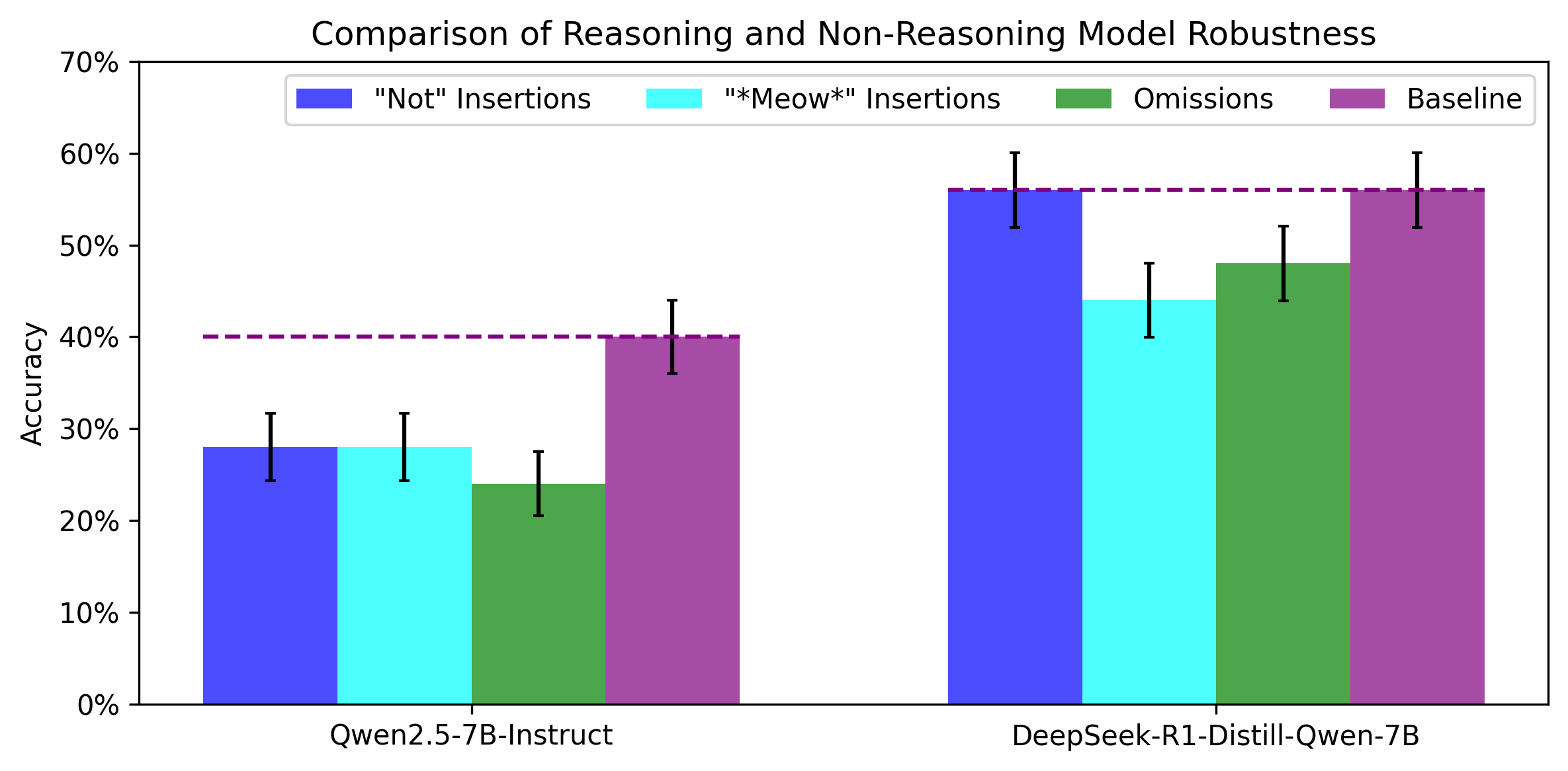}}
  \end{subcaptionbox}
  \hfill
  \begin{subcaptionbox}{\small Distribution of backtracking types identified by Gemini 2.0 Pro in Deepseek R1 responses. The pie chart shows the aggregate counts of each category across responses.\label{fig:deepseek_backtracks_pie}}[0.48\textwidth]
    {\includegraphics[width=\linewidth]{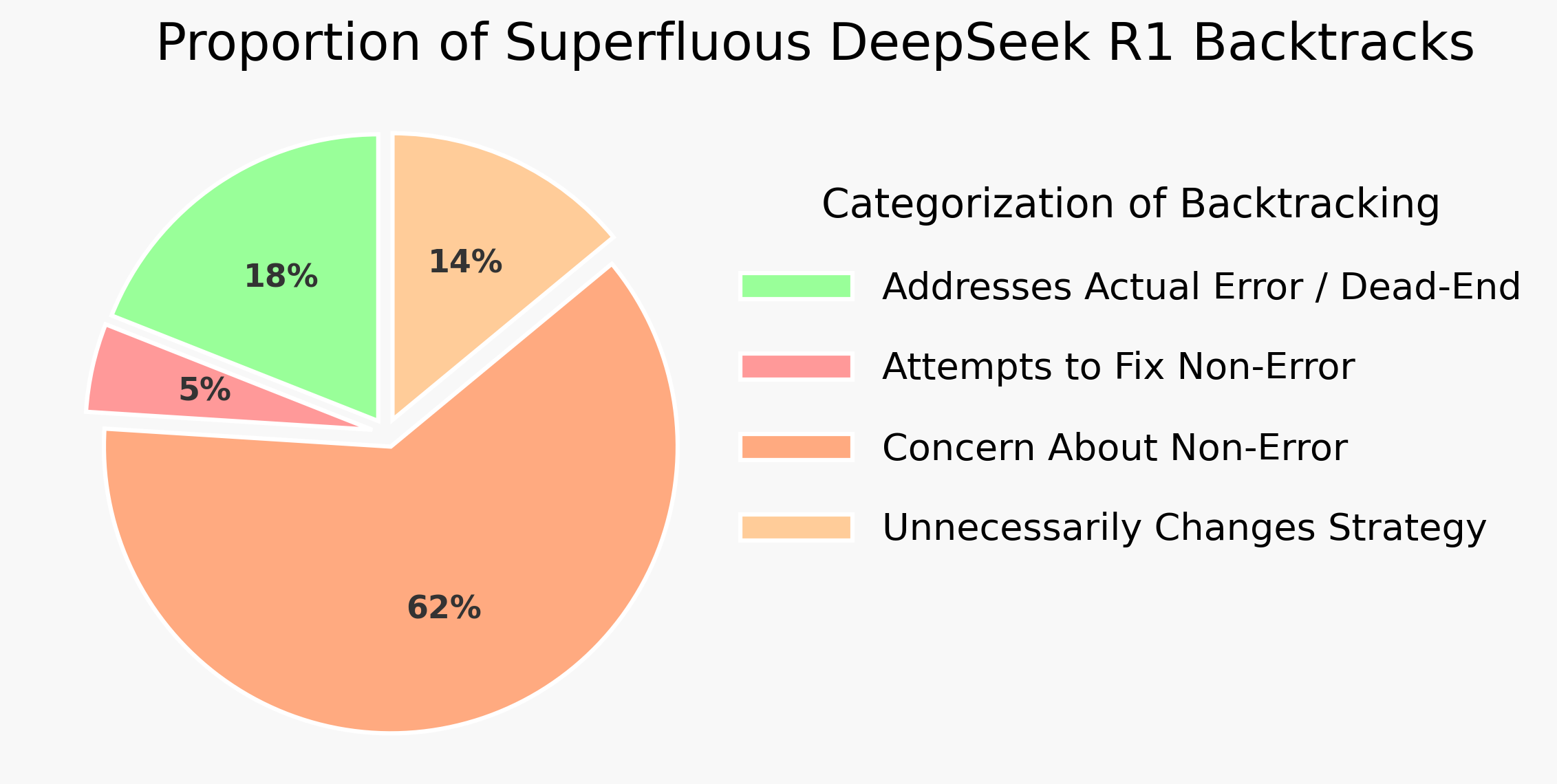}}
  \end{subcaptionbox}
  \caption{Robustness as a reasoning strategy.}
  \label{fig:sidebyside}
\end{figure}

However, our claim is not just that reasoning models exhibit robustness properties (in the sense that they produce correct outputs in the presence of perturbations); rather, we argue that reasoning models have learned robust response \textit{strategies}. 

We work with backtracking as a canonical example of a robust response strategy.
In general, robustness involves paying some trade-offs to mitigate uncertainty;
backtracking allows the model to correct potential intermediate mistakes at the cost of taking longer to arrive at an answer. However, knowing precisely when to backtrack is hard because self-verification is inherently hard.
As a result, reasoning models frequently engage in unnecessary backtracking; in fact, reasoning models will even backtrack when they are relatively confident in their solution---e.g., see Figures \ref{fig:web-of-lies-backtrack} and \ref{fig:livebench-alternatives}. 
Since models cannot be selective in their backtracking behavior, the question of how much to backtrack or self-correct becomes a question of how cautious (i.e., how robust) a model ``wants'' to be. 
{We argue that the \textit{frequency} of unnecessary backtracks indicates that backtracking is a  general strategy learned by reasoning models during post-training. }

In Figure~\ref{fig:deepseek_backtracks_pie}, we analyze the backtracking behavior of Deepseek R1 on a set of reasoning problems again from LiveBench Reasoning, AIME, and MATH \cite{white2024livebenchchallengingcontaminationfreellm, aime2024,DBLP:conf/nips/HendrycksBKABTS21}, with the goal of determining how often backtracking is truly necessary.
However, on these problems, the vast majority of backtracking attempts turn out to serve no purpose at all.   
This result may not be entirely surprising---one might expect that learning to selectively apply a complex reasoning mechanism is much more difficult than learning to aggressively insert pivot tokens somewhat indiscriminately.
In fact, it has been observed that distillation on low-quality examples of reasoning traces still results in improvements
\cite{bespoke_stratos,muennighoff2025s1simpletesttimescaling}: it matters less that the distilled model learns a smart mechanism, and more that the model biases towards frequent backtracking in general.
In the language of our model, RLVR produces a post-trained model whose routing function $g$ heavily favors pretrained circuits that implement backtracking, but, due to finite model capacity (a $\Phi$ of bounded size), is unable to implement a more surgical choice of $(g, \phi)$ that avoids unnecessarily sending queries to backtracking circuits.

\subsection{Preference learning suppresses robustness}
\label{sec:robustness_dispreferred}

This connection between reasoning and robustness provides a language for understanding why reasoning behaviors may be \emph{penalized} by preference data. In particular, observe that if the generation process is reliable or the problem is straightforward, the overhead associated with robustness mechanisms (e.g., backtracking, explicit step-by-step calculations, exploring alternative solutions) can be unnecessary and even detrimental.
Under typical RLHF instructions, labelers may exhibit a bias against these robust reasoning strategies. 
To establish this empirically, we first note that non-reasoning models are able to exhibit reasoning behavior; Figure \ref{fig:existing-robust} below shows one such example in a model not explicitly post-trained for reasoning.

\begin{tcolorbox}[
    colback=gray!10!white,
    colframe=pink!80!black,
    fontlower=\small\sffamily,
    fontupper=\small\sffamily,
    title=Examples of robust behavior already in Gemini 2.0 Flash,
    breakable
]
There are 3 people standing in a line numbered 1 through 3 in a left to right order.
Each person has a set of attributes: Pet, Transport, Food, Nationality.
The attributes have the following possible values:
- Pet: hamster, rabbit, rat
...
- Nationality: japanese, brazilian, indian
and exactly one person in the line has a given value for an attribute.
Given the following premises about the line of people:
- the person that travels by tram is somewhere to the right of the person that has a hamster
...
- the person that likes cranberry is somewhere to the left of the person that travels by tram
- the person that likes watermelon does not own a rabbit or the person that travels by tram doesn't eat watermelon or both
Answer the following question:
At what position is the person that likes avocado? Return a single digit number...

\tcblower

Let's denote the three people as P1, P2, and P3.

1. **The person who is Japanese is on the far left:**  P1 is Japanese.

2. **The person that travels by tram is somewhere to the right of the person that has a hamster:** This means the hamster owner is either P1 or P2.
...

7. **The person that likes watermelon does not own a rabbit or the person that travels by tram doesn't eat watermelon or both:** This is a crucial clue that helps us eliminate possibilities.

Let's analyze possibilities. Since the avocado lover is in P3, and the tram is to the left of the avocado, the tram must be in P2.  The cranberry lover must then be in P1. Since P1 is Japanese, we know the nationality of all three people:

* P1: Japanese, Cranberry
* P2: Tram
* P3: Avocado

{
\color{ForestGreen}
Now let's consider the watermelon and Indian.  Since watermelon is to the right of the tram, it is at P3 and since it's already taken by Avocado. Hence this is a contradiction. But we have watermelon/indian and the tram to deal with.

Considering Point 7 and tram being at P2:  Either watermelon does not own a rabbit (If P3 had watermelon), or the tram does not have watermelon (True). Either way, it does not constrain us to solving a solution at the moment.

Let's go back to the fact that watermelon and Indian have the same parity positions
\dots so there is a contradiction.

However, if we re-examine point 3,  it's not stated watermelon/Indian *have to* be in P3 or P1 necessarily.  
\dots Since tram and avocado at P2 and P3 already eliminate options of placing it in the same positions as well.
} Let's keep the initial arrangement we constructed:

* P1: Japanese, Cranberry
* P2: Tram
* P3: Avocado

Then watermelon is placed in either P1, or P2.  There is a missing piece, as we know there should be the remaining Pet: Rabbit...

Considering the constraints, this fits perfectly, Therefore, Avocado is at position 3.
***3***
\end{tcolorbox}
\noindent\begin{minipage}{\textwidth}
\captionof{figure}{\small Non-reasoning models contain robust behaviors.}
\label{fig:existing-robust}
\end{minipage}

Since non-reasoning models do appear to be capable of robust reasoning strategies, one may hope that RLHF would reinforce such behavior. However, standard preference annotation disprefers and therefore suppresses robust response strategies (prior work makes similar observations, e.g. \cite{yue2025doesreinforcementlearningreally}).
In Figure \ref{fig:rlhf_instructions_comparison}, we show a more systematic study of this phenomenon. 
We select 50 questions from Berkeley MATH \cite{DBLP:conf/nips/HendrycksBKABTS21} and LiveBench \cite{white2024livebenchchallengingcontaminationfreellm}. For each question, we generate a large batch of responses and select two correct answers---one that exhibits backtracking, and one that does not. 
We use Gemini 2.0 Pro to simulate the behavior of a preference annotator, and prompt it to select the preferred response in each pair.
We also experimented between providing our labeler model ``standard'' preference labeling instructions (as explicitly given in \cite{DBLP:conf/nips/Ouyang0JAWMZASR22}) and ``minimal'' labeling instructions (a terse instruction that simply asks for which response is better). 

\begin{figure}[H]
    \centering
    \hfill\includegraphics[width=0.8\textwidth]{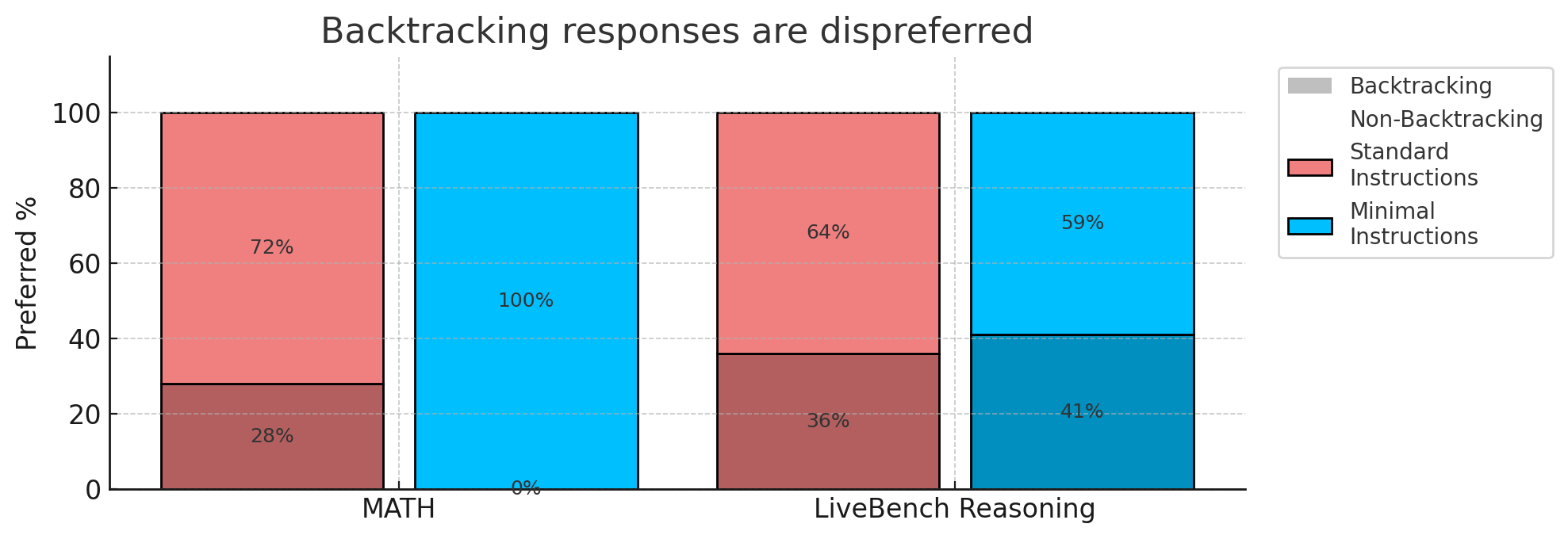}\hspace{0.8cm}
    \caption{\small Comparison of LM (Gemini 2.0 Pro) preferences for succinct, non-backtracking responses versus lengthy, backtracking responses, when both final answers are correct.}
    \label{fig:rlhf_instructions_comparison}
\end{figure}

Across the board, under both variants of labeling instructions, the robustness mechanisms exhibited by the reasoning model (specifically, backtracking) are dispreferred by the preference annotator in a substantial majority of cases. 
This demonstrates that suppression of reasoning is not hypothetical: base models already do possess reasoning behaviors and---unless one's RLHF pipeline is designed to keep an eye out for such behavior (e.g. instructing labelers to specifically reward backtracking), preference datasets will disencourage these behaviors.

In fact, these results almost exactly mimic the scenario illustrated in Example \ref{ex:borda}. The set of circuits $\cS$ learned in pretraining contains both a backtracking strategy $s_\mathrm{Backtrack}$ and a direct strategy $s_\mathrm{Direct}$, which is evidenced by the fact that both styles of correct responses can be generated. However, 
while Section \ref{sec:reasoning_as_robustness} suggests that it may have been ``optimal'' to learn to choose $s_\mathrm{Backtrack}$, Example \ref{ex:borda} predicts and Figure \ref{fig:rlhf_instructions_comparison} confirms that $s_\mathrm{Backtrack}$ often loses out in head-to-head comparisons to $s_\mathrm{Direct}$---which is precisely why the suboptimality in Example \ref{ex:borda} arises.

This also provides a concrete explanation for why the limitations of preference learning we observed in \Cref{theorem:bounded_computation_formal} and \Cref{theorem:bounded_computation_formal_noise} are particularly profound when it comes to RLHF for reasoning. These theorems say that for any preference learning algorithm and any pretrained model, there is a situation where the post-trained model will be lacking.
Robust learning strategies (``compromise candidates'') are useful circuits happen to be exactly the situations that are the main failure mode for the RLHF learning algorithm and---as we saw in this section---reasoning and robustness heavily overlap.

\section{Discussion}
In this work, we seek to develop a conceptual understanding of how preference data as a modality interacts with post-training as a goal.
Our lower bounds in Section \ref{sec:thy} suggest that preference data alone is not enough to make progress in qualitative reasoning tasks.
However, the boundaries of our impossibility result also illuminate pathways to mitigating the lossiness of preference data.
For example, as discussed in Section \ref{sec:thy}, even modest cardinal feedback signals can exponentially decrease information loss \cite{amanatidis2021peeking, DBLP:journals/corr/abs-2009-06543}. 

The question of how to design and calibrate those signals, in a way that is practically impactful, remains an open question, especially as this cardinal data must be an effective proxy for true utility; it is possible that task-specific heuristics may be necessary. For example, if features of effective strategies were known for a particular task, one could explicitly instruct annotators to prioritize them (e.g., for reasoning tasks, labelers could be instructed to prefer strategies that appear to be robust). 
We see these methodological developments as necessary steps for making RL work in domains requiring human feedback. 

More broadly, the categories of tasks we present in Figure \ref{fig:Diagram} were chosen heuristically based on current frontier capabilities. An interesting direction for future work, both theoretically and empirically, is to identify crisper descriptions or categorizations of tasks that we might hope a post-trained model to be able to do. What makes tasks ``similar''? What kinds of data is available, and what kinds of methods can we expect successful? Answering these questions effectively may, perhaps, offer a pathway to sidestep task-specific development towards more principled post-training approaches.

\section{Acknowledgments}
This material is based on work supported by the National Science Foundation Graduate Research Fellowship Program under Grant No. DGE 2146752 (both EZ and JD) and by a Google PhD Fellowship (EZ). The authors thank Kunhe Yang, Nika Haghtalab, and Paul Gölz for valuable conversations and feedback.

\bibliographystyle{alpha}

\begin{thebibliography}{ABFRV21}

\bibitem[ABFRV21]{amanatidis2021peeking}
Georgios Amanatidis, Georgios Birmpas, Aris Filos-Ratsikas, and Alexandros~A Voudouris.
\newblock Peeking behind the ordinal curtain: Improving distortion via cardinal queries.
\newblock {\em Artificial Intelligence}, 296:103488, 2021.

\bibitem[ABFV20]{DBLP:journals/corr/abs-2009-06543}
Georgios Amanatidis, Georgios Birmpas, Aris Filos{-}Ratsikas, and Alexandros~A. Voudouris.
\newblock A few queries go a long way: Information-distortion tradeoffs in matching, 2020.

\bibitem[AFRSV21]{anshelevich2021distortion}
Elliot Anshelevich, Aris Filos-Ratsikas, Nisarg Shah, and Alexandros~A Voudouris.
\newblock Distortion in social choice problems: The first 15 years and beyond, 2021.

\bibitem[AST25]{ahuja2025finding}
Kabir Ahuja, Melanie Sclar, and Yulia Tsvetkov.
\newblock Finding flawed fictions: Evaluating complex reasoning in language models via plot hole detection, 2025.

\bibitem[BCH{\etalchar{+}}12]{boutilier2012optimal}
Craig Boutilier, Ioannis Caragiannis, Simi Haber, Tyler Lu, Ariel~D Procaccia, and Or~Sheffet.
\newblock Optimal social choice functions: A utilitarian view.
\newblock In {\em Proceedings of the 13th ACM Conference on Electronic Commerce}, pages 197--214, 2012.

\bibitem[BKK{\etalchar{+}}22]{bai2022constitutional}
Yuntao Bai, Saurav Kadavath, Sandipan Kundu, Amanda Askell, Jackson Kernion, Andy Jones, and \emph{et al.}
\newblock Constitutional ai: Harmlessness from ai feedback, 2022.

\bibitem[CFH{\etalchar{+}}24]{conitzer2024social}
Vincent Conitzer, Rachel Freedman, Jobst Heitzig, Wesley~H Holliday, Bob~M Jacobs, Nathan Lambert, Milan Moss{\'e}, Eric Pacuit, Stuart Russell, Hailey Schoelkopf, et~al.
\newblock Social choice should guide ai alignment in dealing with diverse human feedback, 2024.

\bibitem[Cit]{google-gemini-deep-research-2024}
Dave Citron.
\newblock Try deep research and our new experimental model in {Gemini}, your {AI} assistant.

\bibitem[CLB{\etalchar{+}}17]{DBLP:conf/nips/ChristianoLBMLA17}
Paul~F. Christiano, Jan Leike, Tom~B. Brown, Miljan Martic, Shane Legg, and Dario Amodei.
\newblock Deep reinforcement learning from human preferences.
\newblock In {\em Advances in Neural Information Processing Systems 30}, pages 4299--4307, 2017.

\bibitem[CLW25]{chakrabarty2025ai}
Tuhin Chakrabarty, Philippe Laban, and Chien-Sheng Wu.
\newblock Ai-slop to ai-polish? aligning language models through edit-based writing rewards and test-time computation, 2025.

\bibitem[CMZ{\etalchar{+}}24]{chen2024preference}
Angelica Chen, Sadhika Malladi, Lily Zhang, Xinyi Chen, Qiuyi~Richard Zhang, Rajesh Ranganath, and Kyunghyun Cho.
\newblock Preference learning algorithms do not learn preference rankings.
\newblock {\em Advances in Neural Information Processing Systems}, 37:101928--101968, 2024.

\bibitem[DF24]{dai2024mappingsocialchoicetheory}
Jessica Dai and Eve Fleisig.
\newblock Mapping social choice theory to rlhf, 2024.

\bibitem[GCS{\etalchar{+}}25]{gandhi2025cognitive}
Kanishk Gandhi, Ayush Chakravarthy, Anikait Singh, Nathan Lile, and Noah~D Goodman.
\newblock Cognitive behaviors that enable self-improving reasoners, or, four habits of highly effective stars, 2025.

\bibitem[GHY25]{yang-private-comm}
Paul Gölz, Nika Haghtalab, and Kunhe Yang.
\newblock Distortion of learning from human feedback: Alignment through implicit utilitarian voting, 2025.

\bibitem[HBK{\etalchar{+}}21]{DBLP:conf/nips/HendrycksBKABTS21}
Dan Hendrycks, Collin Burns, Saurav Kadavath, Akul Arora, Steven Basart, Eric Tang, Dawn Song, and Jacob Steinhardt.
\newblock Measuring mathematical problem solving with the {MATH} dataset.
\newblock In Joaquin Vanschoren and Sai{-}Kit Yeung, editors, {\em Proceedings of the Neural Information Processing Systems Track on Datasets and Benchmarks 1, NeurIPS Datasets and Benchmarks 2021, December 2021, virtual}, 2021.

\bibitem[Lab25]{bespoke_stratos}
Bespoke Labs.
\newblock Bespoke-stratos: The unreasonable effectiveness of reasoning distillation.
\newblock www.bespokelabs.ai/blog/bespoke-stratos-the-unreasonable-effectiveness-of-reasoning-distillation, 2025.

\bibitem[LKB{\etalchar{+}}23]{lightman2023let}
Hunter Lightman, Vineet Kosaraju, Yuri Burda, Harrison Edwards, Bowen Baker, Teddy Lee, Jan Leike, John Schulman, Ilya Sutskever, and Karl Cobbe.
\newblock Let's verify step by step.
\newblock In {\em The Twelfth International Conference on Learning Representations}, 2023.

\bibitem[LLZ{\etalchar{+}}24]{liu2024provably}
Zhihan Liu, Miao Lu, Shenao Zhang, Boyi Liu, Hongyi Guo, Yingxiang Yang, Jose Blanchet, and Zhaoran Wang.
\newblock Provably mitigating overoptimization in rlhf: Your sft loss is implicitly an adversarial regularizer, 2024.

\bibitem[LTC{\etalchar{+}}24]{lai2024step}
Xin Lai, Zhuotao Tian, Yukang Chen, Senqiao Yang, Xiangru Peng, and Jiaya Jia.
\newblock Step-dpo: Step-wise preference optimization for long-chain reasoning of llms, 2024.

\bibitem[MAA24]{aime2024}
Mathematical Association of~America MAA.
\newblock {AIME 2024 Problem Set}, 2024.

\bibitem[Mil]{guardian-chatgpt-creative-writing-2025}
Dan Milmo.
\newblock Chatgpt firm reveals ai model that is ‘good at creative writing’.
\newblock {\em The Guardian}.

\bibitem[MSS{\etalchar{+}}23]{moskovitz2023confronting}
Ted Moskovitz, Aaditya~K Singh, DJ~Strouse, Tuomas Sandholm, Ruslan Salakhutdinov, Anca~D Dragan, and Stephen McAleer.
\newblock Confronting reward model overoptimization with constrained rlhf, 2023.

\bibitem[MYS{\etalchar{+}}25]{muennighoff2025s1simpletesttimescaling}
Niklas Muennighoff, Zitong Yang, Weijia Shi, Xiang~Lisa Li, Li~Fei-Fei, Hannaneh Hajishirzi, Luke Zettlemoyer, Percy Liang, Emmanuel Candès, and Tatsunori Hashimoto.
\newblock s1: Simple test-time scaling, 2025.

\bibitem[OEK{\etalchar{+}}25]{openai2025competitiveprogramminglargereasoning}
OpenAI, :, Ahmed El-Kishky, Alexander Wei, Andre Saraiva, Borys Minaiev, Daniel Selsam, David Dohan, Francis Song, Hunter Lightman, Ignasi Clavera, Jakub Pachocki, Jerry Tworek, Lorenz Kuhn, Lukasz Kaiser, Mark Chen, Max Schwarzer, Mostafa Rohaninejad, Nat McAleese, o3~contributors, Oleg Mürk, Rhythm Garg, Rui Shu, Szymon Sidor, Vineet Kosaraju, and Wenda Zhou.
\newblock Competitive programming with large reasoning models, 2025.

\bibitem[{Ope}]{openai-deep-research-2025}
{OpenAI}.
\newblock Introducing deep research.

\bibitem[Ope24]{o1-preview}
OpenAI.
\newblock {Introducing OpenAI o1-preview}.
\newblock \url{https://openai.com/index/introducing-openai-o1-preview/}, 2024.

\bibitem[OWJ{\etalchar{+}}22]{DBLP:conf/nips/Ouyang0JAWMZASR22}
Long Ouyang, Jeffrey Wu, Xu~Jiang, Diogo Almeida, Carroll~L. Wainwright, Pamela Mishkin, Chong Zhang, Sandhini Agarwal, Katarina Slama, Alex Ray, John Schulman, Jacob Hilton, Fraser Kelton, Luke Miller, Maddie Simens, Amanda Askell, Peter Welinder, Paul~F. Christiano, Jan Leike, and Ryan Lowe.
\newblock Training language models to follow instructions with human feedback.
\newblock In Sanmi Koyejo, S.~Mohamed, A.~Agarwal, Danielle Belgrave, K.~Cho, and A.~Oh, editors, {\em Advances in Neural Information Processing Systems 35: Annual Conference on Neural Information Processing Systems 2022, NeurIPS 2022, New Orleans, LA, USA, November 28 - December 9, 2022}, 2022.

\bibitem[PKD{\etalchar{+}}24]{pal2024smaug}
Arka Pal, Deep Karkhanis, Samuel Dooley, Manley Roberts, Siddartha Naidu, and Colin White.
\newblock Smaug: Fixing failure modes of preference optimisation with dpo-positive, 2024.

\bibitem[PR06]{procaccia2006distortion}
Ariel~D Procaccia and Jeffrey~S Rosenschein.
\newblock The distortion of cardinal preferences in voting.
\newblock In {\em International Workshop on Cooperative Information Agents}, pages 317--331. Springer, 2006.

\bibitem[RMB{\etalchar{+}}24]{razin2024unintentional}
Noam Razin, Sadhika Malladi, Adithya Bhaskar, Danqi Chen, Sanjeev Arora, and Boris Hanin.
\newblock Unintentional unalignment: Likelihood displacement in direct preference optimization, 2024.

\bibitem[RSM{\etalchar{+}}24]{rafailov2024directpreferenceoptimizationlanguage}
Rafael Rafailov, Archit Sharma, Eric Mitchell, Stefano Ermon, Christopher~D. Manning, and Chelsea Finn.
\newblock Direct preference optimization: Your language model is secretly a reward model, 2024.

\bibitem[RSS{\etalchar{+}}25]{ren2025deepseekproverv2advancingformalmathematical}
Z.~Z. Ren, Zhihong Shao, Junxiao Song, Huajian Xin, Haocheng Wang, Wanjia Zhao, Liyue Zhang, Zhe Fu, Qihao Zhu, Dejian Yang, Z.~F. Wu, Zhibin Gou, Shirong Ma, Hongxuan Tang, Yuxuan Liu, Wenjun Gao, Daya Guo, and Chong Ruan.
\newblock Deepseek-prover-v2: Advancing formal mathematical reasoning via reinforcement learning for subgoal decomposition, 2025.

\bibitem[SLHM23]{siththaranjan_distributional_2023}
Anand Siththaranjan, Cassidy Laidlaw, and Dylan Hadfield-Menell.
\newblock Distributional {Preference} {Learning}: {Understanding} and {Accounting} for {Hidden} {Context} in {RLHF}, December 2023.

\bibitem[SNEA{\etalchar{+}}25]{shirali2025directalignmentheterogeneouspreferences}
Ali Shirali, Arash Nasr-Esfahany, Abdullah Alomar, Parsa Mirtaheri, Rediet Abebe, and Ariel Procaccia.
\newblock Direct alignment with heterogeneous preferences, 2025.

\bibitem[SWZ{\etalchar{+}}24]{shao2024deepseekmath}
Zhihong Shao, Peiyi Wang, Qihao Zhu, Runxin Xu, Junxiao Song, Xiao Bi, Haowei Zhang, Mingchuan Zhang, YK~Li, Y~Wu, et~al.
\newblock Deepseekmath: Pushing the limits of mathematical reasoning in open language models, 2024.

\bibitem[Tea25]{deepseekai2025deepseekr1incentivizingreasoningcapability}
DeepSeek-AI Team.
\newblock Deepseek-r1: Incentivizing reasoning capability in llms via reinforcement learning, 2025.

\bibitem[UKK{\etalchar{+}}22]{uesato_solving_2022}
Jonathan Uesato, Nate Kushman, Ramana Kumar, Francis Song, Noah Siegel, Lisa Wang, Antonia Creswell, Geoffrey Irving, and Irina Higgins.
\newblock Solving math word problems with process- and outcome-based feedback, November 2022.

\bibitem[WBZ{\etalchar{+}}22]{DBLP:conf/iclr/WeiBZGYLDDL22}
Jason Wei, Maarten Bosma, Vincent~Y. Zhao, Kelvin Guu, Adams~Wei Yu, Brian Lester, Nan Du, Andrew~M. Dai, and Quoc~V. Le.
\newblock Finetuned language models are zero-shot learners.
\newblock In {\em International Conference on Learning Representations (ICLR)}, 2022.

\bibitem[WDR{\etalchar{+}}24]{white2024livebenchchallengingcontaminationfreellm}
Colin White, Samuel Dooley, Manley Roberts, Arka Pal, Ben Feuer, Siddhartha Jain, Ravid Shwartz-Ziv, Neel Jain, Khalid Saifullah, Siddartha Naidu, Chinmay Hegde, Yann LeCun, Tom Goldstein, Willie Neiswanger, and Micah Goldblum.
\newblock Livebench: A challenging, contamination-free llm benchmark, 2024.

\bibitem[WLC{\etalchar{+}}23]{wang2023making}
Peiyi Wang, Lei Li, Liang Chen, Feifan Song, Binghuai Lin, Yunbo Cao, Tianyu Liu, and Zhifang Sui.
\newblock Making large language models better reasoners with alignment, 2023.

\bibitem[WLJ23]{wang2023rlhf}
Yuanhao Wang, Qinghua Liu, and Chi Jin.
\newblock Is rlhf more difficult than standard rl? a theoretical perspective.
\newblock {\em Advances in Neural Information Processing Systems}, 36:76006--76032, 2023.

\bibitem[WYZ{\etalchar{+}}25]{wang2025reinforcementlearningreasoninglarge}
Yiping Wang, Qing Yang, Zhiyuan Zeng, Liliang Ren, Lucas Liu, Baolin Peng, Hao Cheng, Xuehai He, Kuan Wang, Jianfeng Gao, Weizhu Chen, Shuohang Wang, Simon~Shaolei Du, and Yelong Shen.
\newblock Reinforcement learning for reasoning in large language models with one training example, 2025.

\bibitem[XGZ{\etalchar{+}}24]{xie2024monte}
Yuxi Xie, Anirudh Goyal, Wenyue Zheng, Min-Yen Kan, Timothy~P Lillicrap, Kenji Kawaguchi, and Michael Shieh.
\newblock Monte carlo tree search boosts reasoning via iterative preference learning, 2024.

\bibitem[XHZ{\etalchar{+}}25]{xu2025towards}
Fengli Xu, Qianyue Hao, Zefang Zong, Jingwei Wang, Yunke Zhang, Jingyi Wang, Xiaochong Lan, Jiahui Gong, Tianjian Ouyang, Fanjin Meng, et~al.
\newblock Towards large reasoning models: A survey of reinforced reasoning with large language models, 2025.

\bibitem[YCL{\etalchar{+}}25]{yue2025doesreinforcementlearningreally}
Yang Yue, Zhiqi Chen, Rui Lu, Andrew Zhao, Zhaokai Wang, Yang Yue, Shiji Song, and Gao Huang.
\newblock Does reinforcement learning really incentivize reasoning capacity in llms beyond the base model?, 2025.

\bibitem[YCW{\etalchar{+}}24]{yuan2024advancing}
Lifan Yuan, Ganqu Cui, Hanbin Wang, Ning Ding, Xingyao Wang, Jia Deng, Boji Shan, Huimin Chen, Ruobing Xie, Yankai Lin, et~al.
\newblock Advancing llm reasoning generalists with preference trees, 2024.

\bibitem[YJZ{\etalchar{+}}25]{yu2025benchmarking}
Tong Yu, Yongcheng Jing, Xikun Zhang, Wentao Jiang, Wenjie Wu, Yingjie Wang, Wenbin Hu, Bo~Du, and Dacheng Tao.
\newblock Benchmarking reasoning robustness in large language models, 2025.

\bibitem[ZAG25]{samplingbasedsearch}
Eric Zhao, Pranjal Awasthi, and Sreenivas Gollapudi.
\newblock Sample, scrutinize and scale: Effective inference-time search by scaling verification, 2025.

\bibitem[ZDP{\etalchar{+}}24]{zhang2024chain}
Xuan Zhang, Chao Du, Tianyu Pang, Qian Liu, Wei Gao, and Min Lin.
\newblock Chain of preference optimization: Improving chain-of-thought reasoning in llms.
\newblock {\em Advances in Neural Information Processing Systems}, 37:333--356, 2024.

\bibitem[ZJJ23]{zhu2023principled}
Banghua Zhu, Michael Jordan, and Jiantao Jiao.
\newblock Principled reinforcement learning with human feedback from pairwise or k-wise comparisons.
\newblock In {\em International Conference on Machine Learning}, pages 43037--43067. PMLR, 2023.

\bibitem[ZJJ24]{zhu2024ids}
Banghua Zhu, Michael~I. Jordan, and Jiantao Jiao.
\newblock Iterative data smoothing: Mitigating reward overfitting and overoptimization in rlhf, 2024.

\bibitem[ZNB{\etalchar{+}}25]{zaremba2025trading}
Wojciech Zaremba, Evgenia Nitishinskaya, Boaz Barak, Stephanie Lin, Sam Toyer, Yaodong Yu, Rachel Dias, Eric Wallace, Kai Xiao, Johannes Heidecke, et~al.
\newblock Trading inference-time compute for adversarial robustness, 2025.

\end{thebibliography}
\newcommand{\etalchar}[1]{$^{#1}$}
\newcommand{\nips}[1]{Advances in Neural Information Processing Systems #1}

\newpage

\appendix
\section{Omitted proofs}
\label{appendix:proofs}

For convenience, we will use the shorthand $u_q: \cS \to \reals$ to denote $u_q(s) \assignequals u(q, s(q))$ and overload this notation to denote $u_q(\phi, g, \cS) = u_q((g \circ \phi )(q))$. We will also drop $u$ from the subscripts of $\succ_u$ and $\succ_{q, u}$ when $u$ is  clear from context.

\subsection{Proof of Theorem \ref{theorem:bounded_computation_formal}}

\origdistortion*

We will prove the following generalization of \Cref{theorem:bounded_computation_formal}, which provides precise rates on our distortion bound. \Cref{theorem:bounded_computation_formal} follows immediately by choosing $|\cQ| \gg |\cS|, |\cZ|$ in \Cref{eq:distortionfull}.

\begin{restatable}[Generalization of \Cref{theorem:bounded_computation_formal}]{theorem}{fullorigdistortion}
\label{theorem:bounded_computation_formal_full}
Consider any pretrained model \(M_0 = (\phi_0, g_0, \cS_0)\) and post-training algorithm $\cA$.
There always exists a post-training objective, i.e. a utility $u: \cQ \times \cR \to \reals$ we wish to maximize, such that: if we post-train $M_0$ on noiseless preference data $\succ_u$, the resulting model $M = \cA(M_0, \succ_u)$ is suboptimal by at least a multiplicative factor compared to the best model $M^*$ that we could have post-trained from $M_0$:
\begin{align}
\label{eq:distortionfull}
\underbrace{\max_{M^* \in \{(\phi^*, g^*, \cS_0) \mid \phi^* \in \Phi, g^*: \cZ \to \cS_0\}} \frac{\EEs{q \sim \cD}{u(q, M^*(q))}}{\EEs{q \sim \cD}{u(q, {M}(q))}}}_{(\mathrm{Distortion})}
\geq 
\Omega\Bigg(\min\Bigg\{\sqrt{|\cS_0|}, \frac{|\cQ|}{\sqrt{|\cQ|(\log |\Phi| + |\cZ|)} + |\cZ|}\Bigg\}\Bigg).
\end{align}
Moreover, this lower bound holds even if we limit ourselves to situations where utilities are bounded in $[0, 1]$ or, even stronger, where $\sum_{s \in \cS_0} u(q, s(q)) = 1$ for all $q \in \cQ$.
\end{restatable}
\begin{proof}
Let $m = |\cS_0|$. %
To motivate our multi-color discrepancy argument, we draw an analogy between our post-training setting and a multi-color variant of the balls-and-bins problem.
Specifically, let us analogize queries $\cQ$ to balls $B$, internal representations $\cZ$ to groups (colors) $G$, representation functions $\Phi$ to mappings (colorings) $S \subset G^B$, and query groups $\{\cQ_i\}_i$ to bins $[k]$.
The goal is to argue that there is a way to place our balls in our bins such that for any coloring $g \in S$, the total number of balls belonging to the most popular bin within their own color community is small.

\begin{restatable}{lemma}{balls}\label{lem:coloring_sqrtn_structured}
Let $B$ be a set of $n$ balls, $G$ a set of $r$ groups, and $S \subset G^B$ a set of mappings from balls to groups. There is always a mapping $f$ from balls $B$ to $k$ bins such that for any grouping $g \in S$:
\begin{enumerate}
    \item The total number of balls belonging to the most frequent bin in their respective groups is close to $n/k$:
$$ \sum_{u \in G} \max_{j \in [k]} |\{b \in g^{-1}(u) \mid f(b)=j\}| \leq  O\Bigg(\frac{n}{k} + \sqrt{n r \log k} + r \log k + \sqrt{n \log(|S|)}\Bigg). $$
    \item The total number of balls belonging to the least frequent bin in their respective groups is bounded:
$$ \sum_{u \in G} \min_{j \in [k]} |\{b \in g^{-1}(u) \mid f(b)=j\}| \geq  \Omega\Bigg(\frac{n}{k}\Bigg) - O\Big( \sqrt{n r \log k} +r \log k  + \sqrt{n \log(|S|)}\Big). $$
\end{enumerate}
\end{restatable}

Let us now consider the assignment of queries $\cQ$ into $k$ query groups $\cQ_1, \dots, \cQ_{k}$ that is implied by \Cref{lem:coloring_sqrtn_structured}.
We have the guarantee that for any representation function $\phi \in \Phi$ and internal representation $z \in \cZ$, the function $\phi$ converts a roughly equal number of queries from each query group $\cQ_i$ into the internal representation $z$:
\[\sum_{z \in \cZ} \max_{j \in [k]} |\{q \in \phi^{-1}(z) \mid q \in \cQ_j\}| \leq  O\Bigg(\frac{|\cQ|}{k} + \sqrt{|\cQ| |\cZ| \log k} + |\cZ| \log k  + \sqrt{|\cQ| \log(|\Phi|)}\Bigg).\]
Similarly, for any $\phi \in \Phi$:
\[\sum_{z \in \cZ} \min_{j \in [k]} |\{q \in \phi^{-1}(z) \mid q \in \cQ_j\}| \geq  \Omega\Bigg(\frac{|\cQ|}{k}\Bigg) - O\Big(\sqrt{|\cQ| |\cZ| \log k} +|\cZ| \log k  + \sqrt{|\cQ| \log(|\Phi|)}\Big).\]

We'll now explicitly define an unfortunate post-training objective, which we can recall corresponds to maximizing some utility function $u$.
We start by defining what preferences $\succ_u$ we want our objective $u$ to induce.
Afterwards, we'll finish defining exactly what the objective $u$ is.

\paragraph{Defining $\succ_u$.}
Let us define queries $q \in \cQ_i$ as preferring strategy (i.e., circuit) $s_i$ the most, i.e. $s_i \succ_q s_j$ for all $j \neq i$, and then prefer the remaining strategies $s_j$ in ascending order of index, i.e. $s_j \succ s_{j+1} \succ \dots$.
Because our preference learning algorithm $\cA$ must produce a post-trained model using only access to the pretrained model $M_0 = (\phi_0, g_0, \cS_0)$ and the preferences $\succ_u$, our post-trained model $M = (\phi, g, \cS_0) = \cA(M_0, \succ_u)$ is now well defined.
Recall that, in our notation, $M$ maps queries to internal representations using $\phi \in \Phi$ and uses internal representations to assign queries to circuits via $g \in \cS_0^\cZ$. 

\paragraph{Defining $u$.}
Let us construct a utility function inspired by \citet{boutilier2012optimal}.
We will first define a helper function $v(j) = \tfrac{2}{m-1} (1 - \tfrac{j}{m})$ to help with some tie-breaking nuances. 
We can observe that $v$ is decreasing in its argument, and that $\sum_{j \in [m]} v(j) = 1$ and $v(j) \in [0, 1]$.

Let us next fix any representation $z \in \cZ$ that our post-trained model $M = (\phi, g, \cS_0)$ maps at least one query to, i.e. $\phi^{-1}(z) \neq \emptyset$.
Because $g(z)$ is a probability distribution over the circuits in $\cS_0$, the pigeonhole principle implies that there must exist some circuit index $i_z \in [k]$ such that $g(z)$ places little weight on circuit $s_{i_z} \in \cS_0$: specifically $\Pr(g(z) = s_{i_z}) \leq 1/k$.
We thus have, for each internal representation $z \in \cZ$, a circuit index $i_z \in [k]$ that our post-trained model places a lot of weight on.

Let us still fix a representation $z \in \cZ$ and now consider any query $q$ in our query group $\cQ_{i_z}$ that has the representation $z$, i.e. $q \in \cQ_{i_z} \cap \phi^{-1}(z)$.
We will define our post-training objective for the query $q$, denoted by the utility $u_q$, to be $u_q(s_j) = (1-\epsilon) \cdot \1[s_j = s_{i_z}] + \epsilon \cdot v(j)$, for some small $\epsilon > 0$.
For any query $q$ that also has the representation $z$ but belongs to a different query group $\cQ_i$ where $i \neq i_z$, i.e. $q \in \phi^{-1}(z) \setminus \cQ_{i_z}$, we will define the post-training objective for the query $q$ to be $u_q(s_j) = (1-\epsilon) \cdot \frac{1}{m} + \epsilon \cdot v(j)$.
Observe that these utilities remain consistent with the preferences $\succ_u$, sum to 1 for any given query $\sum_{s \in \cS_0} u_q(S) = 1$, and are bounded in $[0, 1]$.
At a high-level, this means that, for the post-trained model to be high-utility, queries that have the internal representation $z$ should ideally belong to the query group $\cQ_{i_z}$.

\paragraph{Bounding the utility of $M$.}
We now turn to bounding the utility of our post-trained model.
For convenience, let $n_{i, z} = |\{q \in \phi^{-1}(z) \mid q \in \cQ_{i}\}|$ denote the number of queries in query group $\cQ_i$ that share the internal representation $z$.
If we follow our construction above of an unfortunate post-training objective, then, for all representations $z \in \cZ$, the utility of our post-trained model---and in particular its query-to-circuit assignment scheme $g \circ \phi$---can be written as
\begin{align*}
\sum_{z \in \mathrm{Range}(\phi)} 
\sum_{q \in \phi^{-1}(z)} u_q((g \circ \phi)(q))
&\leq
\sum_{z \in \mathrm{Range}(\phi)} 
\Pr(g(z) = s_{i_z})
\cdot
n_{i_z, z}
+
\frac 1 m \sum_{i \neq i_z}
n_{i, z}
\\&\leq
\frac{|\cQ|}{m} + 
\frac{(m-1)\sum_{z \in \mathrm{Range}(\phi)}  n_{i_z, z}}{km}
\end{align*}
where the looseness in the first inequality follows from dropping $\epsilon$ terms.
In other words, the utility of our post-trained model is now directly upper bounded by the number of queries that belong to the query group $\cQ_{i_z}$ that their internal representation $z$ suggests that they should belong to.

To upper bound $n_{i_z, z}$, we will use our discrepancy bound from the first part of \Cref{lem:coloring_sqrtn_structured}, which guarantees that the total number of queries belonging to the ``right'' query group is not that large:
\[
\sum_{z \in \mathrm{Range}(\phi)} 
n_{i_z, z}
\leq  O\Bigg(\frac{|\cQ|}{k} +  \sqrt{|\cQ| |\cZ| \log k} +|\cZ| \log k  + \sqrt{|\cQ| \log(|\Phi|)} \Bigg).
\]
Plugging this back into our upper bound on our post-trained model's utility, we have:
\[
\sum_{q \in \mathcal{Q}} u_q((g \circ \phi)(q))
\leq  O\Bigg( \frac{|\cQ|}{m} + \frac{|\cQ|}{k^2} + \frac{\sqrt{|\cQ| |\cZ| \log k}}{k} + \frac{|\cZ| \log k}{ k} + \frac{\sqrt{|\cQ| \log(|\Phi|)}}{k} \Bigg)
.
\]

\paragraph{Bounding the utility of $M^*$.}
We now want to upper bound the utility of the best possible model $M^*$ that we could have post-trained from our pretrained base $M_0$.
Let us therefore fix any post-trained model $M^* = (\phi^*, g^*, \cS_0)$.
The utility of this model $M^*$ in our constructed post-training objective is simply:
\[
\sum_{z \in \mathrm{Range}(\phi)} 
\sum_{q \in \phi^{-1}(z)} u_q((g^* \circ \phi^*)(q))
\geq
\sum_{z \in \mathrm{Range}(\phi)} 
(1-\epsilon)n_{i_z, z}.
\]
This time, we want to lower bound the number of queries that belong to the query group $\cQ_{i_z}$ that their representations $z$ suggest.
We will therefore use the second part of \Cref{lem:coloring_sqrtn_structured}, which guarantees that
\[
\sum_{z \in \mathrm{Range}(\phi)} 
n_{i_z, z}
\geq \Omega\Bigg( \frac{|\cQ|}{k} \Bigg) - O\Big(\sqrt{|\cQ| |\cZ| \log k} +|\cZ| \log k + \sqrt{|\cQ| \log(|\Phi|)}\Big).
\]
Plugging this back into our upper bound on the hypothetical post-trained model $M^*$'s utility, we have:
\[
\sum_{q \in \mathcal{Q}} u_q((g^* \circ \phi^*)(q))
\geq \Omega\Bigg( \frac{|\cQ|}{k} \Bigg) - O\Big(\sqrt{|\cQ| |\cZ| \log k} + |\cZ| \log k + \sqrt{|\cQ| \log(|\Phi|)}\Big).
\]

\paragraph{Bounding distortion.}
We now prove our main claim by bounding distortion.
Let us fix $k \leq \sqrt{m}$ so $m$ dominates $k^2$.
Let us now adopt the shorthand:\[A = k \cdot O\Big(\sqrt{|\cQ| |\cZ| \log k} + |\cZ| \log k + \sqrt{|\cQ| \log(|\Phi|)}\Big).\]
We want to ensure 
\(
|\cQ|-A \geq \Omega(1)
\), so we can choose:
\[
k = \Theta\left( \frac{|\cQ|}{\sqrt{|\cQ||\cZ| \log |\cQ|} + |\cZ| \log |\cQ| + \sqrt{|\cQ| \log |\Phi|}} \right).
\]
\noindent
We can therefore bound
\begin{align}
\max_{M^* \in \{(\phi^*, g^*, \cS_0) \mid \phi^* \in \Phi, g^*: \cZ \to \cS_0\}} \frac{\EEs{q \sim \cD}{u(q, M^*(q))}}{\EEs{q \sim \cD}{u(q, {M}(q))}} 
&\geq \frac{k(\Omega(|\cQ|) - O(A))}{O(|\cQ| + A)} \nonumber \\
\nonumber
&\geq {\Omega\Bigg(\min\Bigg\{\sqrt{|\cS_0|}, \frac{|\cQ|}{\sqrt{|\cQ|(\log |\Phi| + |\cZ|)} + |\cZ|}\Bigg\}\Bigg).}
\end{align}
\end{proof}

It now remains to prove our \Cref{lem:coloring_sqrtn_structured}. This argument will follow simply from applying the probabilistic approach to classic balls-and-bins high-probability bounds.

\balls*
\begin{proof}[Proof of \Cref{lem:coloring_sqrtn_structured}]
We will begin by constructing a random assignment $f: B \to [k]$ of balls $B$ to bins $[k]$.
We will sample this assignment $f$ by randomly assigning each ball $b \in B$ to a bin $f(b) \in [k]$ uniformly and independently at random.
Let us denote the number of balls in bin $j$ and group $u$ by \[N_{u,j}(f, g) \assignequals |\{b \in g^{-1}(u) \mid f(b)=j\}|.\]
We want to bound two quantities for any given group mapping $g \in S$:
\begin{enumerate}
    \item Letting $M^{\max}_u(f, g) \assignequals \max_{j\in[k]} N_{u,j}(f, g)$ denote the maximum number of balls of any single bin in group $u$, we have \[N_{max}(f, g) \assignequals \sum_{u=1}^r M^{\max}_u(f, g).\]
    \item Letting $M^{\min}_u(f, g) \assignequals \min_{j\in[k]} N_{u,j}(f, g)$ denote the minimum number of balls of any single bin in group $u$, we have: \[N_{min}(f, g) \assignequals \sum_{u=1}^r M^{\min}_u(f, g).\]
\end{enumerate}

We now bound the expectations of $N_{max}(f, g)$ and $N_{min}(f, g)$ for a fixed group mapping $g \in S$.

\begin{claim}\label{claim:expectation_both}
For any fixed $g \in S$, $ \mathbb{E}[N_{max}(f, g)] \le O\Big(\frac{n}{k} + \sqrt{nr \log k} +  r \log k\Big)$ and $ \mathbb{E}[N_{min}(f, g)] \ge \Omega\Big(\frac{n}{k}\Big) - O(\sqrt{nr \log k} + r \log k)$.
\end{claim}
\begin{proof}[Proof of Claim \ref{claim:expectation_both}]
Let $n_u = |g^{-1}(u)|$ be the number of balls in group $u$.
By linearity of expectation, $\mathbb{E}[N_{max}(f, g)] = \sum_{u=1}^r \mathbb{E}[M^{\max}_u(f, g)]$ and $\mathbb{E}[N_{min}(f, g)] = \sum_{u=1}^r \mathbb{E}[M^{\min}_u(f, g)]$.

We can then directly apply the balls-and-bins inequalities for $\mathbb{E}[M^{\max}_u(f, g)] \le O(\tfrac{n_u}{k} + \sqrt{n_u \log k} + \log k)$ and $\mathbb{E}[M^{\min}_u(f, g)] \ge \Omega(\tfrac{n_u}{k}) -  O(\sqrt{n_u \log k} + \log k)$.
Summing over all $r$ groups $u$ with Cauchy-Schwarz gives
\begin{enumerate}
    \item $ \mathbb{E}[N_{max}(f, g)] \le \sum_{u=1}^r O(\tfrac{n_u}{k} + \sqrt{n_u \log k} + \log k) = O(\tfrac{n}{k} + \sqrt{n r \log k} + r \log k)$.
    \item $\mathbb{E}[N_{min}(f, g)] \ge \sum_{u=1}^r \Omega(\tfrac{n_u}{k}) - O(\sqrt{n_u \log k} + \log k) = \Omega(\tfrac{n}{k}) - O(\sqrt{n r \log k} + r \log k)$.
\end{enumerate}
\end{proof}

We now use McDiarmid's inequality to show that the random functions $N_{max}(f, g)$ and $N_{min}(f, g)$ concentrate quickly.
We proceed by first verifying that the functions $f(1), \dots, f(n) \mapsto N_{max}(f, g)$ and $f(1), \dots, f(n) \mapsto N_{min}(f, g)$ satisfy the 1-bounded difference property.

\begin{fact}\label{fact:bounded_diff_both}
If binnings $f$ and $f'$ differ only in how they bin a single ball, then for any $g \in S$, $|N_{max}(f, g) - N_{max}(f', g)| \le 1$ and $|N_{min}(f, g) - N_{min}(f', g)| \le 1$.
\end{fact}
\begin{proof}[Proof of \Cref{fact:bounded_diff_both}]
Let $f$ and $f'$ be two binnings such that they differ only for a single ball $b$. Let $f(b) = j_1$ and $f'(b) = j_2$. Let $u = g(b)$ be the group assigned to ball $b$.
By definition,
$$ N_{max}(f, g) - N_{max}(f', g) = \max_{j \in [k]} N_{u,j}(f, g) - \max_{j \in [k]} N_{u,j}(f', g) \leq  \max_{j \in [k]} N_{u,j}(f, g) - N_{u,j}(f', g) \leq 1. $$
Since the case for $N_{max}(f', g) - N_{max}(f, g) \leq 1$ is symmetric, we have $|N_{max}(f, g) - N_{max}(f', g)| \le 1$.
The proof for $|N_{min}(f, g) - N_{min}(f', g)| \le 1$ follows identically.
\end{proof}

By applying \Cref{fact:bounded_diff_both}, we have from McDiarmid's inequality that for any fixed group mapping $g \in S$ and any deviation tolerance $t > 0$:
\[P(N_{max}(f, g) - \mathbb{E}[N_{max}(f, g)] \ge t) \le \exp\left(\frac{-2t^2}{n}\right),\] and
\[P(N_{min}(f, g) - \mathbb{E}[N_{min}(f, g)] \le -t) \le \exp\left(\frac{-2t^2}{n}\right).\]
Let $\mu_{max}^* = \sup_{g \in S} \mathbb{E}[N_{max}(f, g)]$ and $\mu_{min}^* = \inf_{g \in S} \mathbb{E}[N_{min}(f, g)]$, and consider the ``bad'' event $\mathcal{E}$:
$$ \mathcal{E} = \{ f \mid \exists g \in S: (N_{max}(f, g) > \mu_{max}^* + t) \lor (N_{min}(f, g) < \mu_{min}^* - t) \} $$
Taking a union bound over $g \in S$, we  have
$P(\mathcal{E})  \leq 2 |S| \exp\left(\frac{-2t^2}{n}\right)$.
Let $t = \sqrt{n \log(2 |S|)}$. Then,
$$ P(\mathcal{E}) \le 2 |S| \exp\left(\frac{-2 (n \log(2 |S|))}{n}\right) = \frac{1}{2|S|} < 1. $$
There thus exists at least one binning $f$ for which $\mathcal{E}$ does not occur. This $f$ therefore satisfies both conditions in the lemma claim.
\end{proof}

\subsection{Proof of \cref{theorem:bounded_computation_formal_noise}}

\distortion*

The proof of this theorem follows similarly to that of \Cref{theorem:bounded_computation_formal}.
Once again, we will prove a generalization of \Cref{theorem:bounded_computation_formal_noise} that provides an explicit distortion bound.

\begin{restatable}{theorem}{distortionfull}
\label{theorem:bounded_computation_formal_noise_full}
Consider any pretrained model \(M_0 = (\phi_0, g_0, \cS_0)\) and post-training algorithm $\cA$.
There always exists a post-training objective, i.e. a utility $u: \cQ \times \cR \to \reals$ we wish to maximize, such that the following holds.
Even if we post-train $M_0$ on  preference data $\succ_u$ with noise consistent with the Bradley-Terry model with linear scores, the post-trained model $M = \cA(M_0, \succ_u)$ we obtain is suboptimal by at least a multiplicative factor compared to the best model $M^*$ that we could have post-trained from $M_0$:
\[
\max_{M^* \in \{(\phi^*, g^*, \cS_0) \mid \phi^* \in \Phi, g^*: \cZ \to \cS_0\}} \frac{\EEs{q \sim \cD}{u(q, M^*(q))}}{\EEs{q \sim \cD}{u(q, {M}(q))}} \geq \tilde \Omega\left(\min\left\{ |\cS_0|, \frac{|\cQ|}{\sqrt{|\cQ|(\log |\Phi| + |\cZ|)} + |\cZ|} \right\}\right).
\]
This lower bound holds even if we limit ourselves to situations where utilities are bounded in $[0, 1]$.
\end{restatable}

\begin{proof}
Let us again draw an analogy between queries $\cQ$ and balls $B$,
representations $\cZ$ and groups $G$, representation functions $\Phi$ and mappings $S$, and query groups $\{\cQ_i\}_i$ and bins $[k]$.
This time, we will assign queries $\cQ$ into $k$ query groups $\cQ_1, \dots, \cQ_{k}$.
Once again we will choose these query groups to be those implied by \Cref{lem:coloring_sqrtn_structured}.

Next, we fix some $R \in (0, 1)$ that is close to $1$.
We'll define each query group as inducing the following preferences: queries $q \in \cQ_i$ prefer strategy $s_i$ the most with $\Pr(s_i \succ_q s_j) = R$ for all $j \neq i$.
This restricts $\tfrac{u_q(s_i)}{ u_q(s_j)} = \tfrac{R}{1-R}$.
Note that our preferences $\succ_u$ are now random, since we are in the setting where the outcomes of preference comparisons is random.
One way to think about this is, rather than $\succ_u$ being a proper ordering, to think of $\succ_u$ as an oracle that returns to us the \emph{probability} that one alternative wins over another if one repeated samples noisy preference comparisons between the two.

Now, let $M = (\phi, g, \cS_0) = \cA(M_0, \succ_u)$ denote the model we learned after post-training on these noisy preferences.
Note that, even though we have already fixed $\succ_u$, we have only limited our choice of $u$ to be that where $\tfrac{u_q(s_i)}{ u_q(s_j)} = \tfrac{R}{1-R}$ for all $j \neq i$.
This means that we still have the freedom to set the scale of the maximum attainable utility $\max_{s \in \cS_0} u_q(s)$ for each query $q \in \cQ$.

We will now fix any representation $z \in \cZ$ that has a non-empty pre-image $\phi^{-1}(z) \neq \emptyset$.
Because $g(z)$ is a probability distribution over circuits $\cS_0$, the pigeonhole principle means that there must exist some group index $i_z \in [k]$ where $g(z)$ places little weight on circuit $s_{i_z} \in \cS_0$: $\Pr(g(z) = s_{i_z}) \leq 1/k$.
We will define the maximum utility of any query $q$ in group $\cQ_{i_z}$ with representation $z$ to be $\max_S u_q(S) = 1$.
For queries with representation $z$ but that belong to a different group $\cQ_i$, where $i \neq i_z$, we will instead give them a maximum utility of $\max_S u_q(S) = \epsilon$ for some small $\epsilon$. 

\paragraph{Bounding the utility of $M$.}
Let $n_{i, z} = |\{q \in \phi^{-1}(z) \mid q \in \cQ_{i}\}|$ denote the number of queries that our post-trained model maps to the representation $z$ and that also belong to the correct query group for $z$.
If we follow this construction of our post-training objective $u$ for all $z \in \cZ$, the utility of our post-trained model $M$ can be written as
\[
\sum_{z \in \mathrm{Range}(\phi)} 
\sum_{q \in \phi^{-1}(z)} u_q((g \circ \phi)(q))
\leq
|\cQ| \epsilon + \sum_{z \in \mathrm{Range}(\phi)} n_{i_z, z}
\left( \tfrac{1-R}{R} + \tfrac{1}{k}\right).
\]
To upper bound $n_{i_z, z}$, we will use the first part of \Cref{lem:coloring_sqrtn_structured}, which guarantees that
\[
\sum_{q \in \mathcal{Q}} u_q((g \circ \phi)(q))
\leq \frac{1}{k} \cdot O\Bigg(\frac{|\cQ|}{k} +  |\cZ| \log k  + \sqrt{|\cQ| \log(|\Phi|)} + \sqrt{|\cQ| |\cZ| \log k} + k \cdot |\cQ| \cdot \Big(\epsilon + \tfrac{1-R}{R}\Big)\Bigg).
\]
Letting $\epsilon = O(\tfrac 1 {|\cQ|})$ and $R = \Omega(\tfrac \epsilon {\epsilon + 1})$, the last summand is suppressed to $|\cQ| \cdot (\epsilon + \tfrac{1-R}{R}) \leq O(1)$.

\paragraph{Bounding the utility of $M^*$.}
Consider any possible post-trained model $M^* = (\phi^*, g^*, \cS_0)$ that we could attain.
For each representation $z \in \cZ$, let it route to the ``correct'' circuit $g^*(z) = s_{i_z}$ .
The utility of this model $M^*$ can be written as
\[
\sum_{z \in \mathrm{Range}(\phi)} 
\sum_{q \in \phi^{-1}(z)} u_q((g^* \circ \phi^*)(q))
\geq
\sum_{z \in \mathrm{Range}(\phi)} 
n_{i_z, z}.
\]
Using the second part of \Cref{lem:coloring_sqrtn_structured}, 
\[
\sum_{z \in \mathrm{Range}(\phi)} 
n_{i_z, z}
\geq \Omega\Bigg(\frac{|\cQ|}{k}\Bigg) - O\Bigg(|\cZ| \log k + \sqrt{|\cQ| |\cZ| \log k} + \sqrt{|\cQ| \log( |\Phi|)}\Bigg).
\]

\paragraph{Bounding distortion.}
Let us now adopt the shorthand:
\[
A = k \cdot O(\sqrt{|\cQ| |\cZ| \log k} + |\cZ| \log k + \sqrt{|\cQ| \log(|\Phi|)})
\]
Since we want
\(
|\cQ|-A \geq \Omega(1)
\), we will choose, subject to $k \leq |\cS_0|$, \[
k = \Theta\left( \frac{|\cQ|}{\sqrt{|\cQ||\cZ| \log |\cQ|} + |\cZ| \log |\cQ| + \sqrt{|\cQ| \log |\Phi|}} \right).
\]

\noindent
We can thus lower bound
\begin{align}
\max_{M^* \in \{(\phi^*, g^*, \cS_0) \mid \phi^* \in \Phi, g^*: \cZ \to \cS_0\}} \frac{\EEs{q \sim \cD}{u(q, M^*(q))}}{\EEs{q \sim \cD}{u(q, {M}(q))}} 
&\geq \Omega\left(\frac{k(|\cQ| - A)}{|\cQ| + A + 1}\right) \nonumber \\
\label{eq:eq2}
&\geq { \tilde \Omega\left(\min\left\{ |\cS_0|, \frac{|\cQ|}{\sqrt{|\cQ|(\log |\Phi| + |\cZ|)} + |\cZ|} \right\}\right).} \nonumber
\end{align}
\end{proof} 

\section{Omitted experimental details}
\label{app:expdetails}

Our experiments are performed on the AIME 2024 \cite{aime2024}, Berkeley MATH \cite{DBLP:conf/nips/HendrycksBKABTS21}, and LiveBench (dated October 21, 2024) \cite{white2024livebenchchallengingcontaminationfreellm} datasets, all of which are publicly licensed.
Our experiments are performed on models belonging to the Gemini 2.0 family, including the Gemini-2.0-Flash and Gemini-2.0-Flash-Thinking models.

\subsection{Experimental details for \Cref{fig:model_comparison}.}
For this experiment, we selected five questions from each of the following datasets: AIME 2024, Berkeley MATH test (PRM800K test split), LiveBench Reasoning Zebra Puzzles, LiveBench Reasoning Spatial Reasoning, and LiveBench Reasoning Web-of-Lies-v2 (all LiveBench questions dated October 21, 2024).  We ran both the Qwen 2.5-7B-Instruct and Deepseek-R1-Distill-Qwen-2.5-7B models with a temperature of 0.5 and a maximum token limit of 400,000, terminating upon encountering a stop token.
We introduced three types of perturbations to the model outputs:
\begin{enumerate}
    \item \textbf{``Not'' Insertion:} Every 100 tokens, we inserted the word ``not''.
    \item \textbf{``Meow'' Insertion:} Every 100 tokens, we inserted the phrase ``*\emph{MEOW}*''.
    \item \textbf{Omission:} Every 100 tokens, we deleted five consecutive characters, starting from the fifth character \emph{before} the current head and ending at the tenth character before the head.
\end{enumerate}
For this experiment, the offset for the ``Omissions'' perturbations was chosen to avoid directly deleting newly generated tokens, which would have no effect on the model's output.

In contrast to some existing experiments on the adversarial robustness of reasoning models, e.g. \cite{zaremba2025trading}, this experiment focuses on corruptions to a model's generated chains-of-thought, rather than adversarial manipulations of user queries.

\subsection{Experimental details for \Cref{fig:deepseek_backtracks_pie}.}
For this experiment, we selected two questions from each of the following datasets: 2024 AIME, Berkeley MATH (PRM800K test split), LiveBench Reasoning Zebra Puzzles, LiveBench Reasoning Web-of-Lies-v2, and LiveBench Reasoning Spatial Reasoning.
We use Gemini 2.0 Pro (exp 02-05), with light spot-checking, to identify and categorize instances of backtracking in the Deepseek R1 model's responses. This process consists of:
\begin{enumerate}
    \item Identify all potential instances of backtracking, and determine whether each instance was a genuine case of backtracking.
    \item Analyze the surrounding text to determine if it was \emph{necessary} (i.e., correcting a genuine error).
    \item Subcategorize unnecessary instances as attempting to fix a non-error, expressing concern about a non-error, or unnecessarily changing strategy.
\end{enumerate}
\noindent
In \Cref{fig:deepseek_backtracks_pie}, we categorized unnecessary instances of backtracking into one of three categories:
\begin{itemize}
    \item \textbf{Attempts to Fix Non-Error:} The model attempts to correct a part of its reasoning that was already correct.
    \item \textbf{Concern About Non-Error:} The model expresses doubt or uncertainty about a correct step without identifying a specific error.
    \item \textbf{Unnecessarily Changes Strategy:} The model abandons a valid line of reasoning and adopts a different approach without a clear reason.
\end{itemize}

\begin{tcolorbox}[title=Original Question,breakable]
    In this question, assume each person either always tells the truth or always lies. Luna is at the barbershop... Does the person at the farm tell the truth? Think step by step, and then put your answer in \textbf{bold} as a list of three words, yes or no (for example, \textbf{yes, no, yes}). If you don't know, guess.
    
    \begin{tcolorbox}[title=Category 1: Attempts to Fix Non-Error,
        colback=gray!10!white,
        colframe=red!80!black,
        fontlower=\small\sffamily,
        fontupper=\small\sffamily,breakable]
    ...let me clarify:
    
    - The person at the farm says the person at the amusement park lies.
    
    - The person at the amusement park tells the truth.
    
    Therefore, the farm person's statement is false, meaning the farm person lies.
    
    But earlier, the person at the farm says the person at the amusement park lies, which is true because the amusement park tells the truth. Wait, no. If the farm person says the amusement park lies, and the amusement park tells the truth, then the farm person is lying. So the farm person lies.
    
    ...
    \end{tcolorbox}
    
    \begin{tcolorbox}[title=Category 2: Concern About Non-Error,
        colback=gray!10!white,
        colframe=orange!80!black,
        fontlower=\small\sffamily,
        fontupper=\small\sffamily,breakable]
    ...
    
    Quan: At the zoo, which tells the truth.
    
    Maya: At the farm. The person at the farm says the person at the amusement park lies. But since the amusement park tells the truth, the person at the farm must be lying. So Maya lies.
    
    Now, Maya says the person at the farm lies. Since Maya lies, the opposite is true, so the person at the farm tells the truth. Wait, but earlier I thought the farm person says the amusement park lies, which is true, so the farm person tells the truth. So that's consistent.
    
    ...
    \end{tcolorbox}
    
    \begin{tcolorbox}[title=Category 3: Unnecessarily Changes Strategy,
        colback=gray!10!white,
        colframe=orange!80!white,
        fontlower=\small\sffamily,
        fontupper=\small\sffamily,breakable]
    ...
    
    First, I'll list out all the places and the people associated with them.
    
    - Skate Park: The person there tells the truth.
    ...
    - Airport: The person says the person at the zoo tells the truth.
    - Zoo: The person tells the truth.
    
    Wait, this seems a bit confusing. Let me try to organize this better.
    
    First, identify who is where and what they say.
    
    1. Skate Park: Person tells the truth.
    2. Barbershop: Person tells the truth.
    3. Nia: At the cafe, which tells the truth.
    
    ...
    \end{tcolorbox}
    
\end{tcolorbox}

\subsection{Experimental details for \Cref{fig:rlhf_instructions_comparison}}

For this experiment, we selected 50 questions and answers from the PRM800K split of the Berkeley MATH dataset and 50 questions and answers from the test split of the LiveBench dataset.
For each question, we generated 10 responses from the Gemini 2.0 Flash-001 model at a high temperature (1.9) and a maximum token limit of 8192.
This high temperature was used to encourage diverse responses.
We then scored all 500 responses and filtered out all incorrect answers.
From the remaining responses, we identified a correct answer \emph{without} any instances of backtracking and a correct answer \emph{with} backtracking.
Backtracking instances were identified by using Gemini 2.0 Flash-001 (at temperature 0 and max tokens 8192) to search for phrases indicative of backtracking (e.g., "Wait," "Let me check," "However," etc.).
We then presented these pairs of responses (succinct non-backtracking vs. lengthy backtracking) to Gemini 2.0 Pro, acting as an evaluator, and asked it to choose a preferred response according to one of two possible labeler instructions: the first closely mirrors the RLHF instructions used in \cite{DBLP:conf/nips/Ouyang0JAWMZASR22}; the latter is a minimal set of instructions asking the model to exercise its own judgment.
The results, presented in \Cref{fig:rlhf_instructions_comparison},
strongly indicate a preference for the shorter, non-backtracking responses, even though both responses in each pair were correct.

\section{Additional figures}
\label{app:addfigs}

\subsection{Robustness to Perturbations}
For completeness, we provide below in \Cref{tab:model_accuracy_with_drops} the numbers depicted in \Cref{fig:model_comparison}.

\begin{table}[htbp]
\centering
\caption{Accuracy (with Standard Error) and Percentage Change from Baseline for Model Variants}
\label{tab:model_accuracy_with_drops}
\begin{tabular}{@{}lcc@{}}
\toprule
Model & Qwen2.5-7B-Instruct & DeepSeek-R1-Distill-Qwen-7B \\
\midrule
``Not'' Insertions & \SI{28}{\percent} (\(\pm\)\SI{3.83}{\percent}) \textcolor{red}{[-30\%]} & \SI{56}{\percent} (\(\pm\)\SI{4.05}{\percent}) \textcolor{red}{[0\%]} \\
\addlinespace
``\textit{Meow}'' Insertions & \SI{28}{\percent} (\(\pm\)\SI{3.83}{\percent}) \textcolor{red}{[-30\%]} & \SI{44}{\percent} (\(\pm\)\SI{4.07}{\percent}) \textcolor{red}{[-21\%]} \\
\addlinespace
Omissions & \SI{24}{\percent} (\(\pm\)\SI{3.57}{\percent}) \textcolor{red}{[-40\%]} & \SI{48}{\percent} (\(\pm\)\SI{4.08}{\percent}) \textcolor{red}{[-14\%]} \\
\addlinespace
Baseline & \SI{40}{\percent} (\(\pm\)\SI{4.00}{\percent}) & \SI{56}{\percent} (\(\pm\)\SI{4.05}{\percent}) \\
\bottomrule
\end{tabular}
\end{table}

\subsection{RLHF Simulation}
We further conducted two variations of the experiment depicted in \Cref{fig:rlhf_instructions_comparison}.
In the ``Min-to-Min'' variant, we compared the shortest non-backtracking correct answer against the shortest backtracking correct answer, controlling for response length. Even with this control, we observed a consistent preference for non-backtracking approaches, suggesting that the preference goes beyond simply favoring brevity.
In the ``Thinking'' variant, we replaced Gemini 2.0 Pro with Gemini 2.0 Flash Thinking for our comparisons. Interestingly, this variant showed an even stronger preference for non-backtracking approaches on the LiveBench dataset (74-78\% for non-backtracking compared to only 19\% for backtracking). %

\begin{table}[h]
\centering
\begin{tabular}{@{}llcccc@{}}
\toprule
\multicolumn{1}{c}{\multirow{2}{*}{\textbf{Variant}}} & \multicolumn{1}{c}{\multirow{2}{*}{\textbf{Instructions}}} & \multicolumn{2}{c}{\textbf{LiveBench}} & \multicolumn{2}{c}{\textbf{MATH}} \\
\cmidrule(lr){3-4} \cmidrule(lr){5-6}
 &  & \textbf{BT} & \textbf{Non-BT} & \textbf{BT} & \textbf{Non-BT} \\
\midrule
\multirow{2}{*}{\textbf{Original}} & Standard & 41\% & 52\% & 22\% & 78\% \\
 & Short & 30\% & 70\% & 6\% & 94\% \\
\midrule
\multirow{2}{*}{\textbf{Min-to-Min}} & Standard & 44\% & 56\% & 28\% & 72\% \\
 &  Short & 30\% & 70\% & 22\% & 78\% \\
\midrule
\multirow{2}{*}{\textbf{Thinking}} & Standard & 19\% & 74\% & 28\% & 72\% \\
 & Short & 19\% & 78\% & 22\% & 78\% \\
\bottomrule
\\
\end{tabular}
\caption{Preference comparisons between backtracking (BT) and non-backtracking (Non-BT) responses across different labeler instructions and datasets. We tested three experimental variants: 1) Original - comparing shortest non-backtracking correct vs. longest backtracking correct answers, 2) Min-to-Min - comparing shortest non-backtracking correct vs. shortest backtracking correct answers to control for response length, and 3) Thinking - using Gemini 2.0 Flash Thinking instead of Gemini 2.0 Pro for comparison. Non-backtracking approaches consistently achieve higher accuracy across all variants, particularly with shorter instructions, highlighting the persistent tension between robustness mechanisms and performance.}
\label{tab:accuracy_comparison}
\end{table}

\subsection{Additional Examples of Backtracking}

\begin{tcolorbox}[
    colback=gray!10!white,
    colframe=pink!80!black,
    fontlower=\small\sffamily,
    fontupper=\small\sffamily,
    title=LiveBench Web-of-Lies-v2 Question and Deepseek R1's Response,
    breakable
]
In this question, assume each person either always tells the truth or always lies. The person at the campground says the person at the train station tells the truth... Does the person at the movie theater tell the truth? Does the person at the campground tell the truth? Think step by step, and then put your answer in **bold** as a list of three words, yes or no (for example, **yes, no, yes**). If you don't know, guess.

\tcblower

Okay, let me try to work through this logic puzzle step by step.  There are a lot of statements here, so I need to organize the information carefully. Let's start by listing out all the given facts and statements, and then see how they connect.
First, the people involved and their locations:
Mateo is at the library, Olivia is at the train station, Beatriz is at the cafe, Grace is at the shopping mall, Kehinde is at the movie theater, Charlie is at the campground, and Zara is at the aquarium... says the person at the library (Mateo) lies.
\textcolor{red}{
 \textbf{Campground (Charlie)}: Also, the person at the campground tells the truth. (Wait, this is stated as a fact? Let me check again. The last bullet point says: "The person at the campground tells the truth. Charlie is at the campground." So that's a given fact. So Charlie is a truth-teller. Wait, that's important. So the campground person (Charlie) is a truth-teller. Therefore, his statements are true. So that's a key starting point.
    Wait, let me confirm. The list includes "The person at the campground tells the truth." So that's a fact given in the problem, not a statement made by someone. So Charlie is a truth-teller. Got it.}

\end{tcolorbox}

\noindent\begin{minipage}{\textwidth}
\captionof{figure}{\small Deepseek R1's response to a LiveBench Reasoning question; the model double-checks a stated fact.}
\label{fig:web-of-lies-backtrack}
\end{minipage}

\begin{tcolorbox}[
    colback=gray!10!white,
    colframe=pink!80!black,
    fontlower=\small\sffamily,
    fontupper=\small\sffamily,
    title=LiveBench Math Question and Model Response,
    breakable
]
Triangle $QAC$ has side lengths in arithmetic progression, and the smallest side has length $6.$ If the triangle has an angle of $120^{\circ},$ what is the area of $QAC$?
\\
$\textbf{(A) }20\sqrt{2}\qquad\textbf{(B) }14\sqrt{2}\qquad\textbf{(C) }8\sqrt{6}\qquad\textbf{(D) }15\sqrt{3}\qquad\textbf{(E) }12\sqrt{3}$
\\
If you cannot determine the correct multiple-choice answer, take your best guess. Once you have your answer, please duplicate that letter five times in a single string. For example, if the answer is F, then write FFFFF.

\tcblower

Okay, so I need to solve this problem about triangle QAC with sides in arithmetic progression, the smallest side is 6, and there's a 120-degree angle. The answer choices are given, so I need to figure out the area. Let me start by understanding the problem step by step.

...

If the sides are 6, 10, 14, with the angle opposite 14 being 120 degrees. Then using area formula: $(1/2)*6*10*\sin(120^{\circ}) = 30*(\sqrt{3}/2) = 15*\sqrt{3}$. Yes, that's correct.

\textcolor{green}{
Alternatively, maybe we can compute the area using Heron's formula to verify. Let's try that. Heron's formula states that area = $\sqrt{s(s - a)(s - b)(s - c)}$, where $s$ is the semi-perimeter.}
Compute semi-perimeter $s = (6 + 10 + 14)/2 = 30/2 = 15$.
Then area = $\sqrt{15*(15 - 6)*(15 - 10)*(15 - 14)} = \sqrt{15*9*5*1} = \sqrt{15*45} = \sqrt{675} = \sqrt{225*3} = 15*\sqrt{3}$. So same answer. Therefore, $15\sqrt{3}$ is correct, which is option D.

Therefore, the answer is D, $15\sqrt{3}$.

**Final Answer**

\end{tcolorbox}

\noindent\begin{minipage}{\textwidth}
\captionof{figure}{Deepseek R1's response to a LiveBench Reasoning question. An example of exploring alternative solution paths.}
\label{fig:livebench-alternatives}
\end{minipage}

\section{Prompts}

\subsection{RLHF Instructions}
\begin{tcolorbox}[title=Prompt for Detecting Backtracking, breakable]
You are analyzing a solution to a math problem. Your task is to search for if the solution contains any instances of "backtracking" - where the solver thinks they might have made an error, goes back, and double checks/corrects/redoes a step. This will often be denoted by a comment like "actually," "wait," "hm", "let me reconsider". We expect that only 10\% of solutions will contain backtracking.

\# Math Question

\verb|```|

\textcolor{blue}{\texttt{{Question}}}

\verb|```|

\# Solution

\verb|```|

\textcolor{blue}{\texttt{{Solution}}}

\verb|```|

\# Response Format
Please structure your response as follows:

\verb|```|

\# Your Chain of Thought

\# Evidence

(If "Yes", quote specific parts of the solution that show backtracking. If "No", briefly explain why no backtracking was observed.)

\# Contains Backtracking

(Yes or No)

\verb|```|
\end{tcolorbox}

\begin{tcolorbox}[title=Prompt for Detecting Backtracking Part 2, breakable]
Please structure your output now as JSON, saying nothing else. Use the following format:

\verb|```json|

\verb|{|

\verb|    "contains_backtracking": bool,|

\verb|    "evidence": str|

\verb|}|

\verb|```|
\end{tcolorbox}

\begin{tcolorbox}[title=Standard RLHF Instructions, breakable]
You are given a text-based description of a task, submitted by a user. This task description may be in the form of an explicit math question. You will be given two text outputs, intended to help the user with their task. Your job is to evaluate these outputs to ensure that they are helpful, truthful, and correct. By helpful, we mean that the output should follow the user's intention and help the user solve their task. Some examples of helpful behavior:
• Writing in clear language.
• Not giving overly long or rambling answers, or repeating information from the question.
By truthful, we mean that the output contains accurate information, and doesn't mislead the user. Some examples of truthful
behavior:
• Not producing clearly false information.
• Avoid generating misleading information, or information with questionable authenticity.
By correct, we mean that the output should be mathematically correct and provide a correct answer to the question.
Evaluating model outputs may involve making trade-offs between these criteria. These trade-offs will depend on the task. Use
the following guidelines to help select between outputs when making these trade-offs:
For most tasks, being correct is more important than being helpful. So in most cases, rate an output that's more
correct higher than an output that's more helpful.
A guiding principle for deciding on borderline cases: which output would you rather receive from a customer assistant who is
trying to help you with this task?
Ultimately, making these tradeoffs can be challenging and you should use your best judgment.

\# User Query

\verb|```|

\textcolor{blue}{\texttt{{Query}}}

\verb|```|

\# Solution A

\verb|```|

\textcolor{blue}{\texttt{{Model Response A}}}

\verb|```|

\# Solution B

\verb|```|

\textcolor{blue}{\texttt{{Model Response B}}}

\verb|```|
\end{tcolorbox}

\begin{tcolorbox}[title=Short RLHF Instructions, breakable]
You are performing RLHF data labeling for a language model. Below is a math question and two solutions provided by different models: Solution A and Solution B. Both solutions are correct. Your task is to determine which solution is better.

\# Math Question

\verb|```|

\textcolor{blue}{\texttt{{Query}}}

\verb|```|

\# Solution A

\verb|```|

\textcolor{blue}{\texttt{{Model Response A}}}

\verb|```|

\# Solution B

\verb|```|

\textcolor{blue}{\texttt{{Model Response B}}}

\verb|```|

\# Evaluation

Which solution is better: Solution A or Solution B? Please choose the better solution and briefly explain your reasoning.

\# Response Format

Please structure your response as follows. PROVIDE A COMPLETE RESPONSE.

\verb|```|

\# Better Solution

(Choose either "Solution A" or "Solution B")

\# Reasoning

(Briefly explain why you chose the better solution)

\verb|```|
\end{tcolorbox}

\begin{tcolorbox}[title=RLHF Instructions Part 2, breakable]
Please structure your output now as JSON, saying nothing else. Use the following format:

\verb|```json|

\verb|{|

\verb|    "better_solution": str (either "Solution A" or "Solution B")|

\verb|}|

\verb|```|
\end{tcolorbox}

\subsection{LM-Based Scoring}
\label{app:scoring}

We use the same LM-based scoring as \cite{samplingbasedsearch}.
Given a tuple consisting of a question, ground-truth solution, and candidate response, we grade the correctness of the candidate response by querying a Gemini-v2.0-Flash model to compare the candidate and ground-truth solutions.
This involves providing the question, the correct ground-truth solution, and the candidate response, and asking the model to deliberate on the correctness of the candidate response.
These queries are all processed with temperature zero.
The prompts, which can be found in \cite{samplingbasedsearch}, ask the language model to
(1) identify the final answer of the given response, (2) identify the final answer of the reference (ground truth) response, and (3) determine whether the final answer of the given response satisfactorily matches that of the reference response, ignoring
any non-substantive formatting disagreements.
In line with convention, we instruct our scoring system to ignore the correctness of the logic used to reach the final answer and rather only judge the correctness of the final answer.
The model is asked to label all non-sensical and incomplete responses as being incorrect.

\subsection{Backtracking Detection}
\begin{tcolorbox}[title=Prompt 1, breakable]
I want to count how often my model (who is still rather dumb) backtracks unnecessarily. Your job is to help me with a specific task. Start by helping me by identifying every instance of where the model believes it needs to be cautious about something or might have made an error, and backtracks to either double check or attempt to correct something. Instances of backtracking are usually denoted by "Wait", "Let me check", etc.
1. Examples of backtracking include the model believing it might have made an error and double-checking or attempting to fix it, the model believing its approach is invalid or incorrect or a dead-end and attempting to change its strategy.
2. Expressions of concern that are *not* followed by the model attempting to check or correct something are not backtracks.
Provide a detailed chain of thought for each potential instance of backtracking before ruling on whether it is actually a backtrack.
\end{tcolorbox}

\begin{tcolorbox}[title=Prompt 2, breakable]
For each instance of backtracking that you have found, determine if the backtracking is superfluous. We say that a backtracking is superfluous if any of the following holds:
* The model backtracks because it believes it made an error, and you verify that the model did not actually make an explicit error.
* The model backtracks because it is concerned or cautious about having made an error, and you verify that the model did not actually make an explicit error.
* The model backtracks because it is concerned its current approach is wrong or a dead-end, and you verify that the model's approach is not actually wrong or a dead-end.
If the backtracking is indeed superfluous, then categorize it into one of the following categories:
* Category 1: 'Attempts to Fix Non-Error'
* Category 2: 'Concern About Non-Error'
* Category 3: 'Unnecessarily Changes Strategy'
* Category 4: 'Other'

Question:

\verb|```|

\textcolor{blue}{\texttt{{Question}}}

\verb|```|

Reference Answer:

\verb|```|

\textcolor{blue}{\texttt{{Answer}}}

\verb|```|

Model Response:

\verb|```|

\textcolor{blue}{\texttt{{Model Response}}}

\verb|```|
\end{tcolorbox}

\begin{tcolorbox}[title=Prompt 3, breakable]
Now, provide your final counts in a JSON format, saying nothing else. Structure your output as follows:

\verb|```|

\verb|{"num_not_superfluous": 0, "num_superfluous": 0, "num_category_1": 0,|
\verb|"num_category_2": 0, "num_category_3": 0, "num_category_4": 0}|

\verb|```|
\end{tcolorbox}

\end{document}